\documentclass[12pt]{article}
\usepackage{amsmath}
\usepackage{graphicx,psfrag,epsf}
\usepackage{enumerate}

\usepackage{natbib}
\usepackage{url} % not crucial - just used below for the URL 

\usepackage{amsmath}
\usepackage{amsfonts}
\usepackage{amssymb}
\usepackage{amsthm}
\usepackage{dsfont}
\usepackage[T1]{fontenc}
\usepackage[latin1]{inputenc}
\usepackage{epsfig}
\usepackage{graphicx}
\usepackage{multirow}

\usepackage{tikz}
\usetikzlibrary{matrix,chains,positioning,decorations.pathreplacing,arrows}
\usetikzlibrary{shapes,snakes}
\usetikzlibrary{shapes,arrows}

\usepackage[boxed, lined]{algorithm2e}
\usepackage{float}
\usepackage{caption}
\usepackage{comment}

\newcommand{\blind}{1}

\newcommand{\D}{{\mathcal{D}}}

\newcommand{\B}{{\mathcal{B}}}
\newcommand{\Nu}{{\mathcal{N}}}

\newcommand{\N}{\mathbb{N}}

\newcommand{\R}{\mathbb{R}}
\newcommand{\Z}{\mathbb{Z}}
\newcommand{\Rd}{\mathbb{R}^d}

\newcommand{\beq}{\begin{eqnarray*}}

\newcommand{\eeq}{\end{eqnarray*}}

\newcommand{\beqm}{\begin{eqnarray}}

\newcommand{\eeqm}{\end{eqnarray}}

\newtheorem{theorem}{Theorem}
\newtheorem{corollary}{Corollary}
\newtheorem{lemma}{Lemma}
\newtheorem{definition}{Definition}
\newtheorem{remark}{Remark}

\newcommand{\EXP}{{\mathbf E}}
\newcommand{\PROB}{{\mathbf P}}

\renewcommand{\P}{{\cal P}}
\renewcommand{\bf}{\normalfont \bfseries}
\renewcommand{\it}{\normalfont \itshape}

\newcommand{\F}{{\cal F}}

\allowdisplaybreaks

\begin{document}

\def\spacingset#1{\renewcommand{\baselinestretch}%
{#1}\small\normalsize} \spacingset{1}

%%%%%%%%%%%%%%%%%%%%%%%%%%%%%%%%%%%%%%%%%%%%%%%%%%%%%%%%%%%%%%%%%%%%%%%%%%%%%%

\if1\blind
{
  \title{\bf Estimation of a function of low local dimensionality by deep neural networks}
  \author{Michael Kohler\thanks{
  Funded by the Deutsche Forschungsgemeinschaft (DFG, German
Research Foundation)  - Projektnummer 57157498 - SFB 805.}\hspace{.2cm}\\
    Fachbereich Mathematik, Technische Universit\"at Darmstadt,\\
    Adam Krzy\.zak\thanks{
    The author gratefully acknowledge the support from the Natural Sciences and Engineering Research Council of Canada under grant RGPIW-2015-06412.}\hspace{.2cm}\\
    Department of Computer Science and Software Engineering, \\
    Concordia University
    \\
    and\\
    Sophie Langer$^{*}$\\
    Fachbereich Mathematik, Technische Universit\"at Darmstadt
    }
  \maketitle
} \fi

\if0\blind
{
  \bigskip
  \bigskip
  \bigskip
  \begin{center}
    {\LARGE\bf  Estimation of a function of low local dimensionality by deep neural networks}
\end{center}
  \medskip
} \fi

\bigskip
\begin{abstract}
Deep neural networks (DNNs) achieve impressive results for complicated tasks like object detection on images and speech recognition. Motivated by this practical success, there is now a strong interest in showing good theoretical properties of DNNs. To describe for which tasks DNNs perform well and when they fail, it is a key challenge to understand their performance. The aim of this paper is to contribute to the current statistical theory of DNNs.\\
We apply DNNs on high dimensional data and we show that the least squares regression estimates using DNNs are able to achieve dimensionality reduction in case that the regression function has locally low dimensionality. Consequently, the rate of convergence of the estimate does not depend on its input dimension $d$, but on its local dimension $d^*$ and the DNNs are able to circumvent the curse of dimensionality in case that $d^*$ is much smaller than $d$. In our simulation study we provide numerical experiments to support our theoretical result and we compare our estimate with other conventional nonparametric regression estimates. The performance of our estimates is also validated in experiments with real data. 

\end{abstract}
\noindent%
{\it Keywords:}  curse of dimensionality,
deep neural networks,
nonparametric regression,
piecewise partitioning,
rate of convergence.
\vfill

\newpage
\spacingset{1.45} % DON'T change the spacing!
\section{Introduction}
\label{sec:intro}
 \subsection{Nonparametric regression}
 Motivated by the huge success of deep neural networks
in applications (cf., e.g.,
\cite{Sch15} and the literature cited therein)
there is now keen interest in investigating theoretical
properties of deep neural networks. In 
statistical research this is usually done in context
of nonparametric regression (cf., \cite{KoKr17}, \cite{BK17}, \cite{Sch17}, \cite{KL20}, \cite{FM18} and \cite{NI19}).  
% Nonparametric regression is one of the standard frameworks to analyze the theoretical properties of DNNs 
%We show these results
%in the context of nonparametric regression with random
%design. 
Here,
$(X,Y)$ is an $\Rd \times \R$--valued random vector
satisfying $\EXP \{Y^2\}<\infty$, and given a sample
of $(X,Y)$ of size $n$, i.e., given a data set
\begin{equation}
  \label{se1eq1}
\D_n = \left\{
(X_1,Y_1), \dots, (X_n,Y_n)
\right\},
\end{equation}
where
$(X,Y)$, $(X_1,Y_1)$, \dots, $(X_n,Y_n)$ are i.i.d. random variables,
the aim is to construct an estimate
\[
m_n(\cdot)=m_n(\cdot, \D_n):\Rd \rightarrow \R
\]
of the regression function $m:\Rd \rightarrow \R$,
$m(x)=\EXP\{Y|X=x\}$ such that the $L_2$ error
\[
\int |m_n(x)-m(x)|^2 \PROB_X (dx)
\]
is ``small'' (see, e.g., \cite{GKKW02}
for a comprehensive study to nonparametric regression and
 motivation for the $L_2$ error).

 \subsection{Rate of convergence}
It is well--known that one
needs smoothness assumptions on the regression function in
order to derive non--trivial rates of convergence
(cf., e.g., Theorem 7.2 and Problem 7.2 in
\cite{DGL96} and
Section 3 in \cite{DW80}).
Thus we introduce the following definition.
\begin{definition}
\label{intde2}
  Let $p=q+s$ for some $q \in \N_0$ and $0< s \leq 1$,
where $\N_0$ is the set of nonnegative integers.
A function $f:\R^d \rightarrow \R$ is called
\textbf{$(p,C)$-smooth}, if for every $\alpha=(\alpha_1, \dots, \alpha_d) \in
\N_0^d$
with $\sum_{j=1}^d \alpha_j = q$ the partial derivative
$\frac{
\partial^q f
}{
\partial x_1^{\alpha_1}
\dots
\partial x_d^{\alpha_d}
}$
exists and satisfies
\[
\left|
\frac{
\partial^q f
}{
\partial x_1^{\alpha_1}
\dots
\partial x_d^{\alpha_d}
}
(x)
-
\frac{
\partial^q f
}{
\partial x_1^{\alpha_1}
\dots
\partial x_d^{\alpha_d}
}
(z)
\right|
\leq
C
\cdot
\| x-z \|^s
\]
for all $x,z \in \R^d$, where $\Vert\cdot\Vert$ denotes the Euclidean norm.
\end{definition}
\cite{Sto82} showed that the optimal minimax rate of convergence in nonparametric
regression for $(p,C)$-smooth functions is $n^{-2p/(2p+d)}$.

 \subsection{Curse of dimensionality}
In case
that $d$ is large compared to $p$ the above rate of convergence is
rather slow which is a symptom of so-called curse of dimensionality. One way to circumvent it is to
impose additional constraints on the structure of the regression
function. Recently it was shown, that deep neural networks are able to circumvent the curse of dimensionality whenever suitable hierarchical composition assumptions on the regression function hold. Here the regression function is contained in the following function class:

\begin{definition}
\label{de2}
Let $d \in \N$ and $m: \Rd \to \R$ and let
$\P$ be a subset
of $(0,\infty) \times \N$

\noindent
\textbf{a)}
We say that $m$ satisfies a hierarchical composition model of level $0$
with order and smoothness constraint $\mathcal{P}$, if there exists a $K \in \{1, \dots, d\}$ such that
\[
m(x) = x^{(K)} \quad \mbox{for all } x = (x^{(1)}, \dots, x^{(d)})^{\top} \in \Rd.
\]
\noindent
\textbf{b)}
We say that $m$ satisfies a hierarchical composition model
of level $\ell+1$ with order and smoothness constraint $\mathcal{P}$, if there exist $(p,K)  \in \P$, $C>0$, $g: \R^{K} \to \R$ and $f_{1}, \dots, f_{K}: \Rd \to \R$, such that
$g$ is $(p,C)$--smooth,
$f_{1}, \dots, f_{K}$ satisfy a hierarchical composition model of level $\ell$
with order and smoothness constraint $\mathcal{P}$
and 
\[m(x)=g(f_{1}(x), \dots, f_{K}(x)) \quad \mbox{for all } x \in \Rd.\]
\end{definition}

In case that the order and smoothness constraint of $g$ alternates between $(p,d)$ and $(\infty, K)$ and $g$ is a sum in every second level, this definition equals the definition of the so--called \textit{$(p,C)$-smooth generalized hierarchical interaction models} of order $d^*$, which were introduced by \cite{KoKr17}. They showed that for such models suitably defined multilayer
neural networks (in which the number of hidden layers depends
on the level of the generalized interaction model) achieve the rate of convergence  $n^{-2p/(2p+d^*)}$
(up to some logarithmic factor) in case $p \leq 1$.
\cite{BK17} generalized this result for $p>1$ provided the squashing function is suitably
chosen. For the hierarchical composition model of Definition \ref{de2}, where the smoothness and dimension is fixed within one level, \cite{Sch17} showed (up to some logarithmic factor) a rate of convergence 
\begin{align*}
\max_{(p,K) \in \P} n^{-2p/(2p+K)}
\end{align*}
for sparse neural networks with ReLU activation function. \cite{KL20} showed that this rate holds even for simple fully connected neural networks and arbitrary hierarchical composition model of Definition \ref{de2}. All the above mentioned results are optimal up to some logarithmic factor. \cite{Liu19} showed that some of these results hold even without the logarithmic factor. For regression functions with a form of common statistical models, i.e. multivariate adaptive regression splines (MARS), \cite{EcSch18} showed that convergence rate by DNNs can also be improved. 
%that in case that the regression function e
%\cite{EcSch18} showed a connection between neural networks with ReLU activation function and linear spline-type methods.  In particular, they showed that every function expressed as a function in MARS can also be approximated by a multilayer neural network (up to a sup-norm error $\epsilon$). Using this result they derived a risk comparison inequality, that bound the statistical risk of fitting a neural network by the statistical risk of spline-based methods.
In case that the regression function is defined on a manifold, \cite{Sch19} showed, that the convergence rate by DNNs depends on the dimension of the manifold. \cite{NI19} analyzed the performance of DNNs in case that the high-dimensional data have an intrinsic low dimensionality and showed that the convergence rate by DNNs depends only on the intrinsic dimension and not on the input dimension.

\subsection{Low local dimensionality}
In this article we consider regression functions with low local dimensionality. There exist several examples in the literature, where high dimensional problems can be treated locally in much lower dimension. \cite{Bell_1997} showed that the probability distribution of a natural scene is highly structured, since, for instance, the neighboring pixel of a natural scene have redundant informations. \cite{ViSch97} and \cite{HSV09} analyzed in their research on human motor control some regularities in full-body movement of humans within and across individuals. These regularities also lead to locally low-dimensional data distributions. For instance, they showed that for estimating the inverse dynamics of an arm, a globally 21- dimensional space reduces, on average to 4-6 dimensions locally. And also in our own research it can be reasonably assumed, that the analyzed data set is of a locally low dimensional structure. The data set under study (which is part of the Machine Learning Repository: \url{https://archive.ics.uci.edu/ml/machine-learning-databases/00275/}) is related to 2--year usage log of a bike sharing system namely Captial Bike Sharing (CBS) at Washington, D.C., USA (\cite{FG13}). The data show the hourly aggregated count of rental bikes and 12 attributes, namely the season (1: spring, 2: summer, 3: fall, 4: winter), the year (0: 2011, 1:2012), the month (1 to 12), the hour (0 to 23), holiday (whether the day is holiday (1) or not (0)), the day of the week (1 to 7), workingday (if day is neither weekend nor holiday is 1, otherwise is 0), the weather situation (1: Clear, Few clouds, Partly cloudy, 2: Mist + Cloudy, Mist + Broken clouds, Mist + Few clouds, Mist, 3: Light Snow, Light Rain + Thunderstorm + Scattered clouds, Light Rain + Scattered clouds, 4: Heavy Rain + Ice Pallets + Thunderstorm + Mist, Snow + Fog), the normalized temperature in Celsius, the normalized feeling temperature in Celsius, the normalized humidity and the normalized windspeed. For this data set we conjecture that depending on the season,
the hour and the attribute working day the count of rental
bikes depends on different subsets of the other attributes.
E.g., in spring and fall during the rush hour on working
day the weather is not important at all. But on days which
are not working days, it depends mainly on the hour and the
weather, where for different seasons different weather
attributes are important (like temperature and humidity
in summer or weather situation in the Spring and in the Fall).
%
%
%
%
%We observe that depending on the season the count of the rental bikes depends differently on the weather situation, the temperature, the humidity and windspeed. For instance, in summer and winter the count of the rental bikes mainly depends on the normalized temperature, whereas in spring and fall it is affected by the weather situation. All other components describing the weather (i.e. humidity, windspeed etc.) are in both situations redundant.  
%Additionally there is some redundancy in the dataset, due to the fact that some attributes correlate with each other, for example, the holiday and the workingday attribute. 
This leads to the assumption, that the underlying regression function performs differently on different subsets and depends locally only on a few of its input components. 
%
% Thus we observe that the data locally depend only on a very few of its input components, which leads to the assumption, that the data set also has some locally low dimensional structure.
  In summary, that means that one can reduce dimension locally without losing much information for many high dimensional problems thus avoiding the curse of dimensionality. This finding motivates us to analyze regression functions with low local dimensionality. 
%
%
%the time of the day and whether it is a working day or a holiday the count of the rental bikes depends differently on the weather situation, the temperature, the humidity and windspeed. For instance, during a working day and during peak times (in the morning and the late afternoon), the count of the rental bikes is affected less by all these components. In contrast, during holidays the weather situation is the important factor for the count of the rental bikes. In both situations we observe that the count of the rental bikes depends only on a few, but different, input components. 
%This leads to the assumption, that the data set also has some locally low dimensional structure. In summary, that means that one can reduce dimension locally without losing much information for many high dimensional problems thus avoiding therefore the curse of dimensionality. This finding motivates us to analyze regression functions with low local dimensionality. 
%that the regression function performs differently on different subsets. Combining this with the redundancy we observe some low local dimensionality in the data set, which also fits to our assumption on the regression function. 
\\
\\
%In summary, there is a class of high dimensional problems, where the local dimension of the data is much lower.
We say 
%In the sequel we want to show that Theorem \ref{th1} implies
%that neural networks can achieve a dimensionality reduction
%in case of a $(p,C)$-smooth regression function with a low local dimensionality.
%Here
a function $f:\Rd \rightarrow \R$ has a low local dimensionality, if it depends locally only on a very few of its component, where in different areas these subsets of variables can be different. 
The simplest
way to define this formally is to assume that there exist
$d^* \in \{1, \dots, d\}$,
$K \in \N$, disjoint sets $A_1$, \dots, $A_K \subset \Rd$, functions
$f_1$, \dots, $f_{K}:\R^{d^*} \rightarrow \R$ and subsets of indices
$J_1$, \dots, $J_K \subset \{1, \dots, d\}$ of cardinality at most
$d^*$ such that
\begin{equation}
\label{se3eq1}
f(x)
=
\sum_{k=1}^K
f_k(x_{J_k}) \cdot \mathds{1}_{
A_k
}(x)
\end{equation}
holds for all $x \in \Rd$, where 
\begin{align*}
x_{\{j_{k,1}, \dots, j_{k,d^*}\}} = (x^{(j_{k,1})}, \dots, x^{(j_{k,d^*})}) \ \mbox{for} \ 1 \leq j_{k,1} < \dots < j_{k,d^*} \leq d.
\end{align*}
%And in case that the $f_k$ are all $(p,C)$--smooth we want to show
%that for such a regression function we can achieve the rate of convergence
%\[
%n^{-\frac{
%2 p
%}{
%2 p + d^*
%}}.
%\]
As a consequence of using the indicator function, assumption (\ref{se3eq1})
implies that $f$ in general is not globally smooth, in particular it
is not even continuous. In view of many applications where it is
intuitively
expected that the dependent variable depends smoothly on the
independent variables, this does not seem to be realistic.

To avoid this problem, we will allow in the sequel smooth transitions 
between the different areas $A_1$, \dots, $A_K$
in (\ref{se3eq1}). To achieve this, we assume that the function
$f$ is squeezed between two functions of the form (\ref{se3eq1}). In order to simplify the presentation, we use in the sequel $d$-dimensional polytopes for the sets $A_1, \dots, A_K$. Since polytopes can be described as the intersection of a finite number of half spaces, we define the local dimensionality as follows:
\begin{definition}
\label{se3de2}
A function $f: \R^d \to \R$ has \textit{local dimensionality} $d^{\ast} \in \{1, \dots, d\}$ on $[-A,A]^d$ for $A > 0$ with order $(K_1, K_2)$, $\PROB_X$-border $\epsilon >0$ and borders $\delta_{i,k} > 0$ for $i =1, \dots, K_1$, $k=1, \dots, K_2$ if there exists $a_{i,k} \in \R^d$ with $\Vert a_{i,k} \Vert \leq 1$, $b_{i,k} \in \R$, $J_k \subseteq \{1, \dots, d\}$ with $|J_k| \leq d^{\ast}$ for  $i =1, \dots, K_1$, $k=1, \dots, K_2$ and 
\begin{equation*}
f_k: \R^{d^{\ast}} \to \R
\end{equation*}
such that for 
\begin{equation*}
(P_k)_{\delta_k} = \{x \in \R^d: a_{i,k}^T x \leq b_{i,k} - \delta_{i,k} \ \mbox{for $i = 1, \dots, K_1$}\}
\end{equation*}
and 
\begin{equation*}
(P_k)^{\delta_k}= \{x \in \R^d: a_{i,k}^T x \leq b_{i,k} + \delta_{i,k} \ \mbox{for $i = 1, \dots, K_1$}\}
\end{equation*}
with $\delta_k = (\delta_{1,k}, \dots, \delta_{K_1,k})$
we have
\begin{equation*}
\sum_{k=1}^{K_2} f_k(x_{J_k}) \cdot 1_{(P_k)_{\delta_k}}(x) \leq f(x) \leq \sum_{k=1}^{K_2} f_k(x_{J_k}) \cdot 1_{(P_k)^{\delta_k}}(x) \ \ \ (x \in A)
\end{equation*}
and 
\begin{equation*}
\PROB_X\left(\left( \bigcup_{k=1}^{K_2}(P_k)^{\delta_k} \textbackslash (P_k)_{\delta_k}\right)\cap [-A,A]^d\right) \leq \epsilon.
\end{equation*}
\end{definition}
Figure \ref{fig2} shows a function $f(x)$ with $K_2=4$, polytopes $P_1 = [-2,0] \times [-2,0]$, $P_2=[-2,0] \times [0,2]$, $P_3 = [0,2] \times [0,2]$ and $P_4 = [0,2] \times [-2,0]$ and functions $f_1(x_1)=\sin(4 \cdot x_1)$, $f_2(x_2)=\exp(x_2)$, $f_3(x_2) = \cos(4 \cdot x_2)$ and $f_4(x_2) = \exp(x_1)$ with smooth transitions between the polytopes. 
\begin{figure}[h]
\centering
\includegraphics[scale=0.5]{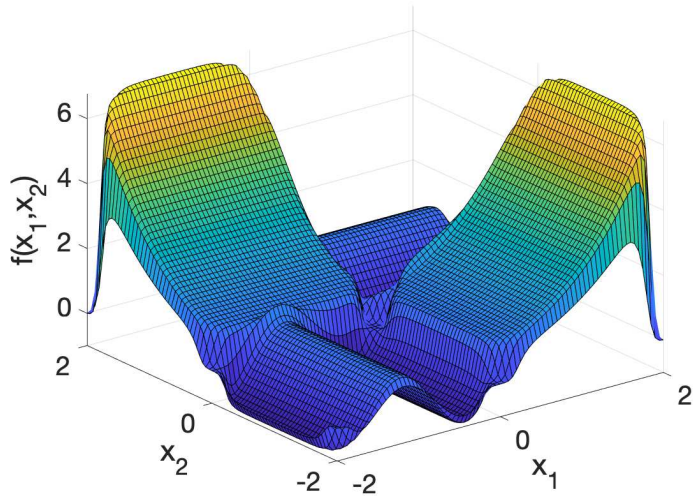}
\caption{An example of a function with low local dimensionality with a 2-dimensional support. The support is divided into four pieces, the function depends locally only on one variable of the input and is globally smooth.}
\label{fig2}
\end{figure}

\subsection{Main results}
In this paper we show that sparse neural network regression estimates are able
to achieve a dimension reduction in case that the regression
function has a low local dimensionality. We derive the rate of convergence which depends 
only on the local dimension $d^*$ and not on the input dimension $d$ (cf., Theorem \ref{th2}). 
Thus our neural network regression
estimates are able to circumvent the curse of dimensionality in case that $d^*$ is much smaller than $d$. Finally, we verify the theoretical results using simulation studies and experiments on real data.
\\
\\
 We point out another advantage of DNNs, namely that that neural networks are able to detect a locally low dimensional structure and therefore achieve a faster rate of convergence. As a technical contribution of this paper, we present a result concerning the connection between neural networks and so-called multivariate adapative regression splines (MARS). For instance, we show that the sparse neural network
regression estimates, where the weights are chosen by the least squares,
satisfy the expected error bound of MARS in case that this procedure works in an optimal way (cf., Theorem 2 in the Supplement).
\\
Our results are based
on a set of sparse neural networks instead of fully connected neural networks. On the one hand 
this network architecture leads to a better bound of the covering number, which is essential to
show the convergence result. On the other hand they perform better with regard to the simulated and real data 
as shown in our simulation studies. In applying our estimates to a real-world data experiment we 
emphasize the practical relevance of our assumption on the regression function and show
that our sparse neural network estimates outperform other nonparametric regression estimates, especially MARS, on this data set. 
\subsection{Discussion of related results}
  It is frequently observed by various researchers, that the true intrinsic dimensionality of high dimensional data is often very low (e.g. \cite{BN03}, \cite{T00}, \cite{HSV09}, \cite{NI19}). Several estimators like the kernel methods and the Gaussian process regression are able to exploit the intrinsic low dimensionality of covariates and achieve a fast rate of convergence depending only on the intrinsic dimensionality of the data set (e.g. \cite{B07}, \cite{S11}, \cite{K13}, \cite{Yang:1997aa}). Recently, \cite{NI19} also derived convergence rates by DNNs, which only depend on the intrinsic dimension and are free from the nominal dimension. \cite{Sch19} achieved approximation rates of DNNs that only depend on $d^*$, in case that the input lies on on a $d^*$-dimension manifold.
%  
%  For inputs on a $d^*$-dimension manifold,  \cite{Sch19} achieved approximation rates of DNNs that only depend on $d^*$. 
  To describe the intrinsic dimensionality, both articles used the notion of \textit{Minkowski dimension}. 
  \\
  \\
All the above mentioned results use the observation, that many high dimensional problems are contained in a low dimensional space. At this point we would like to highlight the difference with our assumption. While these studies focus on the behavior of the measure $\PROB_X$ of covariates, we analyze regression functions with some specific structure, i.e. regression functions with low local dimensionality. In our case the dimension of the regression function is locally of size $d^* \leq d$ and the regression function performs differently on different subsets. This does not imply, that there is some intrinsic low dimensionality in the measure $\PROB_X$ of covariates. 
  \\
  \\
A similar structure of regression functions has been studied by \cite{FM18}. They analyzed the performance of DNNs for a certain class of piecewise smooth functions. Here piecewise smooth regression functions where the partitions have smooth borders were considered. For instance, their partition consists of a finite number of pieces, where each piece is an intersection of so-called \textit{basis pieces}. Each basis piece is defined with the help of a horizon function and is regarded as one side of surfaces by a H\"older-smooth function. Thus the pieces of the partition in this paper have smooth borders, which is a more flexible way to define piecewise smooth functions, but which does not contain the case of globally smooth functions. Since we also want to take into consideration smooth functions with low local dimensionality, i.e. functions which perform differently on different pieces (depending only on a few components of the input on each piece), but are nevertheless globally smooth, we define our pieces as d--dimensional polytopes and allow smooth transition between them.
\\
\\
As mentioned earlier, the proof of our main result is based on a result that analyzes the connection between DNNs and MARS. \cite{EcSch18} already showed a similar result for the ReLU activation function. In particular, they showed that every function expressed as a function in MARS can also be approximated by a multilayer neural network (up to a sup-norm error $\epsilon$). Using this result they derived a risk comparison inequality, that bound the statistical risk of fitting a neural network by the statistical risk of spline-based methods.
Due to the fact that the ReLU activation function and consequently 
the corresponding neural network are piecewise linear functions 
it is not that suprisingly to find connection to spline methods. This paper extends this result by showing connection between neural networks with smooth 
activation functions and MARS, which was not covered by the results in \cite{EcSch18}. Additionally, we 
show our result for a more general basis of smooth piecewiese polynomials, i.e. a product of a
truncated power basis of degree 1 and a B-spline basis. This leads to better approximation properties 
in case of very smooth regression function. \\
\\
The approximation of B-Splines by DNNs has also been studied by \cite{Liu19}. They showed that a DNN with $const \cdot m$ hidden layers and a fixed number of neurons per layer achieves an approximation rate of size $4^{-2m}$ for a tensor product B-spline basis. In the Supplement we derive a related result for DNN with squashing activation function. 
%
%
%They showed a similar approximation result as ours (see Appendix) 
%\cite{Liu19} presented a similar approximation result with neural networks using the ReLU activation function. Since in this result $x$ is contained in the interval $[0,1]$ and the activation function is non-smooth, we could not use it in our approximation. 

%The proof of our main result is based 
   %\cite{EcSch18} showed a connection between neural networks with ReLU activation function and linear spline-type methods.  In particular, they showed that every function expressed as a function in MARS can also be approximated by a multilayer neural network (up to a sup-norm error $\epsilon$). Using this result they derived a risk comparison inequality, that bound the statistical risk of fitting a neural network by the statistical risk of spline-based methods.
%\newline
%\newline
%\cite{Liu19} analyzed the approximation of tensor product B-splines by multi-layer networks with ReLU activation function. They showed that a fully connected deep neural network approximates each tensor product B-spline with an error of size $4^{-m}$, where the number of hidden layers of the network depends on $m$. 
 \subsection{Notation}
Throughout the paper, the following notation is used:
The sets of natural numbers, natural numbers including $0$, integers, non-negative real numbers and real numbers
are denoted by $\N$, $\N_0$, $\Z$, $\R_+$ and $\R$, respectively. For $z \in \R$, we denote
the smallest integer greater than or equal to $z$ by
$\lceil z \rceil$, and
$\lfloor z \rfloor$
denotes the largest integer that is less than or equal to $z$.
Furthermore we set $z_+=\max\{z,0\}$.
The Euclidean and the supremum norms of $x \in \Rd$
are denoted by $\|x\|_2$ and $\|x\|_\infty$, respectively.
For $f:\R^d \rightarrow \R$
\[
\|f\|_\infty = \sup_{x \in \R^d} |f(x)|
\]
is its supremum norm, and the supremum norm of $f$
on a set $A \subseteq \R^d$ is denoted by
\[
\|f\|_{\infty,A} = \sup_{x \in A} |f(x)|.
\]
A finite collection $f_1, \dots, f_N:\Rd \rightarrow \R$
  is called an $\varepsilon$--$\|\cdot\|_{\infty,A}$-- cover of $\F$
  if for any $f \in \F$ there exists  $i \in \{1, \dots, N\}$
  such that
  \[
  \|f-f_i\|_{\infty,A}
  =
  \sup_{x \in A} |f(x)-f_i(x)| < \varepsilon.
  \]
  The $\varepsilon$--$\|\cdot\|_{\infty,A}$- covering number of $\F$
  is the  size $N$ of the smallest $\varepsilon$--$\|\cdot\|_{\infty,A}$-- cover
  of $\F$ and is denoted by $\Nu(\varepsilon,\F,\| \cdot \|_{\infty,A})$.
We write $x = \arg \min_{z \in D} f(z)$ if
$\min_{z \in \D} f(z)$ exists and if
$x$ satisfies
\[
x \in D \quad \mbox{and} \quad f(x) = \min_{z \in \D} f(z).
\]
  If not otherwise stated, any $c_i$ with $i\in\N$
  symbolizes a real nonnegative constant,
  which is independent of the sample size $n$.

\subsection{Outline}
The outline of this paper is as follows: In Section \ref{se2}
the sparse neural network regression estimates analyzed in this
paper are introduced. The main result is presented in Section
\ref{se3}. The finite sample size behavior of our estimate is analyzed by applying it to simulated and real data in Section \ref{se4}. 
%Section \ref{se5} contains the proofs.

\section{Sparse neural network regression estimates}
\label{se2}
The starting point in defining a neural network is the
choice of an activation function \linebreak $\sigma: \mathbb{R} \to \mathbb{R}$.
We use in the sequel
so--called squashing functions, which are nondecreasing
and satisfy $\lim_{x \rightarrow - \infty} \sigma(x)=0$
and
$\lim_{x \rightarrow  \infty} \sigma(x)=1$.
An example of a squashing function is
the so-called sigmoidal or logistic squasher
\begin{equation}
  \label{se2eq1a}
\sigma(x)=\frac{1}{1+\exp(-x)} \quad (x \in \R).
\end{equation}
A multilayer feedforward neural network with
$L$ hidden layers and $k_1$, $k_2$, \dots, $k_L$
number of neurons in the first,
second, $\dots$, $L$-th hidden layer
and sigmoidal function $\sigma$ is a real-valued function defined on $\mathbb{R}^d$ of the form
\begin{equation}\label{se2eq1}
f(x) = \sum_{i=1}^{k_L} c_{1,i}^{(L)} \cdot f_i^{(L)}(x) + c_{1,0}^{(L)},
\end{equation}
for some $c_{1,0}^{(L)}, \dots, c_{1,k_L}^{(L)} \in \mathbb{R}$ and for $f_i^{(L)}$'s recursively defined by
\begin{equation}
  \label{se2eq2}
f_i^{(r)}(x) = \sigma\left(\sum_{j=1}^{k_{r-1}} c_{i,j}^{(r-1)} \cdot f_j^{(r-1)}(x) + c_{i,0}^{(r-1)} \right)
\end{equation}
for some $c_{i,0}^{(r-1)}, \dots, c_{i, k_{r-1}}^{(r-1)} \in \mathbb{R}$
$(r=2, \dots, L)$
and
\begin{equation}
  \label{se2eq3}
f_i^{(1)}(x) = \sigma \left(\sum_{j=1}^d c_{i,j}^{(0)} \cdot x^{(j)} + c_{i,0}^{(0)} \right)
\end{equation}
for some $c_{i,0}^{(0)}, \dots, c_{i,d}^{(0)} \in \mathbb{R}$.
We denote
by $\F(L,r,\alpha)$
the set of all fully connected neural networks with $L$ hidden layers, $r$
neurons
in each hidden layer and weights bounded in absolute value by $\alpha$. 
\newline
In the sequel we propose sparse neural networks architectures, where the consecutive layers of neurons are not fully connected. The structure of our sparse neural networks depends on smaller neural networks that are fully connected. For $M^{*} \in \N$, $L \in \N$, $r \in \N$ and $\alpha > 0$, we denote the set of all functions $f: \R^d \to \R$ that satisfy
\begin{equation*}
f(x) = \sum_{i=1}^{M^{*}} \mu_i \cdot f_i(x) \ (x \in \R^d)
\end{equation*}
for some $f_i \in \mathcal{F}(L, r, \alpha)$ and for some $\mu_i \in \R$, where $|\mu_i| \leq \alpha$, by $\mathcal{F}^{(sparse)}_{M^{*}, L, r, \alpha}$. An example of a network in class $\mathcal{F}^{(sparse)}_{3, L, r, \alpha}$ is shown in Figure \ref{fig1} which gives a good idea of how the network structure looks like. 
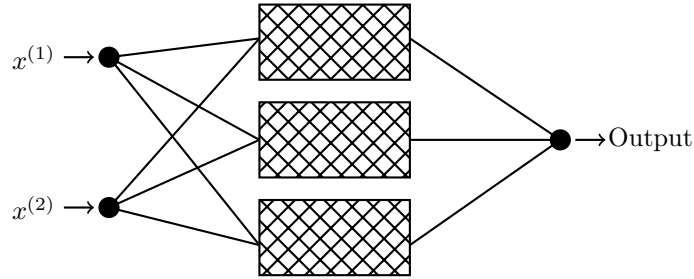
\begin{figure}[h]
\centering
 \begin{tikzpicture}[thick] 
     \tikzstyle{annot} = [
    text centered
    ]
 \tikzstyle{rectangle_style}=[rectangle, draw]
 		\node[annot] at (1, -3+2) {\footnotesize{$x^{(2)}$}};
 		\node[annot] at (1, -3+4) {\footnotesize{$x^{(1)}$}};
 		\draw[->] (1.4,-3+2) -- (1.8,-3+2);
 		\draw[->] (1.4,-3+4) -- (1.8,-3+4);
 		\draw[->] (8.2,2-2.1) -- (8.6,2-2.1);
 		\node[annot] at (9.2, 2-2.1) {\footnotesize{Output}};
 %\node[draw] at (2,-3) {Input};
 \foreach \x in {1,2}
		\fill (2, -3 + 2*\x) circle (4pt);

 %Box1
% \begin{scope}
%   \clip[draw] (4,2) rectangle (6,3); 
%    \draw[rotate=45,step=2mm] (-5,-5) grid (7,7); 
%    \end{scope}
 % \foreach \x in {1,2,3}
  \begin{scope}
    \clip[draw] (4,2-1.3*1) rectangle (6,3-1.3*1);
      \draw[rotate=45,step=2mm] (-5,-5) grid (7,7); 
    \end{scope}
     \begin{scope}
    \clip[draw] (4,2-1.3*2) rectangle (6,3-1.3*2);
      \draw[rotate=45,step=2mm] (-5,-5) grid (7,7); 
    \end{scope}
     \begin{scope}
    \clip[draw] (4,2-1.3*3) rectangle (6,3-1.3*3);
      \draw[rotate=45,step=2mm] (-6,-6) grid (7,7); 
    \end{scope}
  %  \draw[rotate=45,step=2mm] (-4,-4) grid (4,4); 
% \end{scope}
    \draw (2,-3+4) -- (4, 2-0.75);
    \draw (2,-3+4) --(4,  2-2.1);
    \draw (2,-3+4) --(4,  2-3.5);
    
      \draw (2,-3+2) -- (4, 2-0.75);
    \draw (2,-3+2) --(4,  2-2.1);
    \draw (2,-3+2) --(4,  2-3.5);
    
    \fill (8, 2-2.1) circle (4pt);
    
    \draw (6, 2-0.75) --(8, 2-2.1);
    \draw (6, 2-2.1) -- (8, 2-2.1);
    \draw (6, 2-3.5) -- (8, 2-2.1);

    %\draw[rotate=45,step=2mm] (10,0) grid (11,2); 

%    Box2
%  \draw (4,-1) rectangle (6,0); 
%    \draw[rotate=45,step=2mm] (-7,-7) grid (7,7); 
  \end{tikzpicture}
  \caption{A neural network with $M^*=3$ boxes of fully connected neural networks}
  \label{fig1}
  \end{figure}
In the sequel we want to use data (\ref{se1eq1}) in order
to choose a function from $\mathcal{F}^{(sparse)}_{M^{*}, L, r, \alpha}$ such that this function is
a good regression estimate. In order to do this, we use the principle
of
least squares and define our regression estimate $\tilde{m}_n$ as a
function
\begin{equation}
\label{se2eq4}
\tilde{m}_n(\cdot)= \tilde{m}_n(\cdot,\D_{n}) \in \F^{(sparse)}_{M^{*}, L, r, \alpha_n}
\end{equation}
 from $\F^{(sparse)}_{M^{*}, L, r, \alpha_n}$, which minimizes the so--called
empirical $L_2$ risk over $\F^{(sparse)}_{M^{*}, L, r, \alpha_n}$, i.e., which satisfies
\begin{equation}
\label{se2eq5}
\frac{1}{n} \sum_{i=1}^n
| Y_i - \tilde{m}_n(X_i)|^2
=
\min_{ f \in \F^{(sparse)}_{M^{*}, L, r, \alpha_n}}
\frac{1}{n} \sum_{i=1}^n
| Y_i - f(X_i)|^2.
\end{equation}
 Here we assume for notational simplicity that the minimum above does
indeed exists. In case that it does not exist our results also hold for
any function chosen from $\F^{(sparse)}_{M^{*}, L, r, \alpha_n}$ which minimizes
the empirical $L_2$ risk in (\ref{se2eq5}) up to some small additive
term, e.g., up to $1/n$.
 For technical reasons in the analysis of our estimate we need to
 truncate
it at some data--independent level $\beta_n$ satisfying $\beta_n
\rightarrow \infty$ for $n \rightarrow\infty$, i.e., we set
\begin{equation}
\label{se2eq6}
m_n(x)=T_{\beta_n} \tilde{m}_n(x) \quad (x \in \Rd),
\end{equation}
 where $T_{\beta_n} z = \max \{
\min\{ z, \beta_n \}, - \beta_n
\}$ for $z \in \R$.

The number $L$ of  layers and the number $r$ of parameters of each fully connected neural network $f_i$ will be chosen as a large enough constant.  For the bound
$\alpha_n$ on the absolute value of the weights we will use a
data--independent
bound of the form $\alpha_n=c_1 \cdot n^{c_2}$ for some $c_1,c_2>0$. The main
parameter
left which controls the flexibility of the networks is then the number
$M^{\ast}$ of fully connected neural networks $f_i \in \F(L,r,\alpha_n)$ $(i=1, \dots, M^{*})$. To choose it, we will use the principle of
splitting of the sample (cf., e.g., Chapter 7 in Gy\"orfi et al. (2002)).
Here we split the sample into a learning sample of size $n_l$ and a
testing sample of size $n_t$, where $n_l,n_t \geq 1$ satisfy
$n=n_l+n_t$,
e.g., $n_l= \lceil n/2 \rceil$ and $n_t=n-n_l$. We use the learning
sample
\[
\D_{n_l}= \left\{
(X_1,Y_1), \dots, (X_{n_l},Y_{n_l})
\right\}
\]
to define for each $M^{*}$ in $\P_n=\{ 2^l \, : \, l=1, \dots, \lceil \log
n \rceil \}$
an estimate $\tilde{m}_{n_l,M^{*}}$ by
\begin{equation}
\label{se2eq7}
\tilde{m}_{n_l,M^{*}}(\cdot)=\tilde{m}_{n_l,M^{*}}(\cdot,\D_{n_l}) \in \F^{(sparse)}_{M^{*}, L, r, \alpha_n}
\end{equation}
and
\begin{equation}
\label{se2eq7b}
\frac{1}{n_l} \sum_{i=1}^{n_l}
| Y_i - \tilde{m}_{n_l,M^{*}}(X_i)|^2
=
\min_{ f \in \F^{(sparse)}_{M^{*}, L, k, \alpha_n}}
\frac{1}{n_l} \sum_{i=1}^{n_l}
| Y_i - f(X_i)|^2,
\end{equation}
and set
\begin{equation}
\label{se2eq8}
m_{n_l,M^{*}}(x)=T_{\beta_n} \tilde{m}_{n_l,M^*}(x) \quad (x \in \Rd).
\end{equation}
Then we choose
$M^{*} \in \P_n$ such that the empirical $L_2$ error of the estimate
on the testing data is minimal, i.e., we define
\begin{equation}
\label{se2eq9}
m_n(x,\D_n)=m_{n_l,\hat{M}^{*}}(x,\D_{n_l}),
\end{equation}
 where
\begin{equation}
\label{se2eq10}
\hat{M^{*}} \in \P_n
\quad \mbox{and} \quad
\frac{1}{n_t} \sum_{i=n_l+1}^n
| Y_i - m_{n_l, \hat{M^{*}}}(X_i)|^2
=
\min_{M^* \in \P_n}
\frac{1}{n_t} \sum_{i=n_l+1}^n
| Y_i - m_{n_l,M^*}(X_i)|^2.
\end{equation}

\section{Main result}
\label{se3}

Our theoretical result will be valid for sigmoidal functions which
are
$2$--admissible according to the following definition.
\begin{definition}
  \label{se3de1}
  Let $N \in \N_0$.
  A function $\sigma : \mathbb{R} \to [0, 1]$ is called \textbf{N-admissible},
  if it is nondecreasing and Lipschitz continuous and if, in addition,
  the following three conditions are satisfied:
\begin{itemize}
\item[(i)] The function $\sigma$ is $N+1$ times continuously differentiable with bounded derivatives.
\item[(ii)] A point $t_{\sigma} \in \mathbb{R}$ exists, where all
  derivatives up to  order $N$ of $\sigma$ are
 nonzero.
\item[(iii)] If $y > 0$, the relation $|\sigma(y) - 1| \leq \frac{1}{y} $ holds. If $y < 0$, the relation $|\sigma(y)| \leq \frac{1}{|y|}$ holds. \label{3}
\end{itemize}
\end{definition}
\noindent
It is easy to see that the logistic squasher (\ref{se2eq1a})
is $N$--admissible for any $N \in \N$ (cf., e.g. \cite{BK17}).

Our main result shows, that the sparse neural networks can achieve the
$d^*$--dimensional rate of convergence in case that the regression function has local
dimensionality $d^*$.
\begin{theorem}
  \label{th2}
  Let $\beta_n = c_{3} \cdot \log(n)$ for some constant $c_{3}>0$. Assume that the distribution of $(X,Y)$ satisfies
\begin{align}\label{subgaus}
\mathbf E \left(\exp({c_{4}\cdot |Y |^2})\right) <\infty
\end{align}
for some constant $c_{4}>0$ and that the distribution of $X$ has bounded
support $supp(X) \subseteq [-A,A]^d$ for some $A \geq 1$. Let $M, K_1,K_2 \in \N$.
 Assume furthermore that $m$
has local dimensionality $d^*$ on $supp(X)$ with order $(K_1, K_2)$, $\PROB_X$-border $c_{5}/n$ and borders $\delta_{i,k}$, where $\delta_{i,k} \geq c_{6}/n^{c_{7}}$ holds for some constants $c_{5}, c_{6}, c_{7} > 0$ ($i =1, \dots, K_1$, $k=1, \dots, K_2$) and where all functions
$f_k$ in Definition \ref{se3de2} are bounded and $(p,C)$--smooth for some $p = q+s$ with $0< s \leq 1$ and $q \leq M$.

Let the least squares neural network
regression estimate $m_n$ be defined as in Section \ref{se2}
with parameters $L=3K_1+d \cdot (M+2)-1$, $r=2^{M-1} \cdot 16+ \sum_{k=2}^M 2^{M-k+1} +d +5$, $\alpha_n=c_1 \cdot n^{c_2}$  and $n_l=\lceil  n / 2 \rceil$.
 Assume
 that the sigmoidal function $\sigma$ is $2$--admissible, and
 that $c_1,c_2>0$ are suitably large.
Then we have for any $n > 7$:
\begin{eqnarray*}
  &&
  \EXP \int |m_n(x)-m(x)|^2 \PROB_X(dx)
 \leq
c_8 \cdot
  (\log n)^3 \cdot 2^{K_1} \cdot K_2 \cdot
n^{
- \frac{2p}{2p+d^*}.
}
\end{eqnarray*}
  \end{theorem}
  \noindent
  The proof is available in the Supplement.

\noindent

\begin{remark}
The class of regression functions with low local dimensionality $d^*$ satisfying the assumptions of Theorem \ref{th2} contains all $(p,C)$-smooth functions, which depend at the most on $d^*$ of its input components. This is because the polytopes in the definition of low local dimensionality (see Definition \ref{se3de2}) can be chosen as one single hyperplane $a^T x \leq b$ ($K_1=K_2=1$) with $\Vert a \Vert \leq 1$, $a \in \R^{d^*}$ and $b= \sqrt{d \cdot A} +1$, in which case the single hyperplane contains all $x \in [-A,A]^d$. Consequently, the rate of convergence in Theorem \ref{th2} is optimal up to some logarithmic factor according to \cite{Sto82}.
\end{remark}

\begin{remark}
The deep neural network estimate in the above theorem achieves a rate
of convergence which is independent of the dimension $d$ of $X$,
hence it is able to circumvent the curse of dimensionality in case
that the regression function has low local dimensionality.
\end{remark}

\begin{proof}[Outline of the proof of Theorem \ref{th2}]
In the proof of Theorem \ref{th2} the following bound on the expected $L_2$ error of our sparse neural network regression estimate is essential:
\begin{eqnarray}
\label{th2eq1}
  &&
  \EXP \int |m_n(x)-m(x)|^2 \PROB_X(dx)
 \leq
  (\log n)^3 \cdot
\inf_{\substack{I \in \N, \ B_1, \dots, B_I \in \mathcal{B}^{*}}}
\Bigg(
c_9 \cdot \frac{I}{n}
\notag  \\
  &&
\hspace*{3cm}
+
\min_{(a_i)_{i=1, \dots, I}\in [-c_{10} \cdot n, c_{10} \cdot n]^I}
\int
|\sum_{i=1}^I a_{i} \cdot B_{i}(x)-m(x)|^2 \PROB_X(dx)
\Bigg).
\end{eqnarray}
Here $\mathcal{B}^*$ is a basis consisting of functions representable as a product of a truncated power basis of degress 1, i.e. the MARS function class, and a tensor product B-spline basis (see the Supplement for a detailed definition). A complete proof of this bound can be found in Theorem 2 in the Supplement. In the proof we derive some approximation-theoretical properties of sparse DNNs. For instance we show that our sparse DNNs approximate functions of the form 
\begin{align*}
B(x) = \sum_{i=1}^I a_{i} \cdot B_{i}(x), \quad B_i \in \mathcal{B}^*.
\end{align*}
Since every function with low local dimensionality $d^*$ (according to Definition \ref{se3de2}) can be expressed as a linear combination of functions of $\mathcal{B}^*$ in case that $x$ is not contained in $\left(\left(\bigcup_{k=1}^{K_2}(P_k)^{\delta_k} \textbackslash (P_k)_{\delta_k}\right)\cap [-A,A]^d\right)$, we can use the bound \eqref{th2eq1} to show our main result. Here we proceed as follows: First we show that an indicator function of a polytope can be approximated by a linear truncated power basis. In the second step we prove that every $(p,C)$-smooth function can be approximated by a linear combination of a tensor product B-Spline basis. In the last step we show that every function of the form
\begin{align*}
\sum_{k=1}^{K_2} f_k(x_{J_k}) \cdot \mathds{1}_{(P_k)^{\delta_k}}(x)
\end{align*}
with notations according to Definition \ref{se3de2} can be expressed as a linear combination of functions of $\mathcal{B}^*$. Together with the assumption
\begin{align*}
\PROB_X\left(\left(\bigcup_{k=1}^{K_2}(P_k)^{\delta_k} \textbackslash (P_k)_{\delta_k}\right)\cap [-A,A]^d\right) \leq \frac{c_{5}}{n}
\end{align*}
we conclude the assertion of the Theorem. 
%Using the bound of \eqref{th2eq1} the main idea of the proof of Theorem \ref{th2} is, that we can express every function with low local dimensionality $d^*$ according to Definition \ref{se3de2} by a linear combination of functions of $\mathcal{B}^*$ in case that $x$ is not contained in $\left(\left(\bigcup_{k=1}^{K_2}(P_k)^{\delta_k} \textbackslash (P_k)_{\delta_k}\right)\cap [-A,A]^d\right)$.
\end{proof}

\section{Simulation study}
\label{se4}
To illustrate how the introduced nonparametric regression estimate based on our sparsely connected neural networks behaves in case of finite sample sizes, we apply it to simulated data using the MATLAB software. Due to the fact that our estimate contains some parameters that may influence their behavior, we will choose these parameters in a data-dependent way by splitting of the sample. Here we use $n_{train}=\lceil \frac{4}{5} \cdot n \rceil$ realizations to train the estimate several times with different choices for the parameters and $n_{test}=n-n_{train}$ realizations to test the estimate by comparing the empirical $L_2$ risk of different parameter settings and choosing the best estimate according to this criterion. The parameters $L$, $r$ and $M^{*}$ of the estimates in Section \ref{se2} are chosen in a data-dependent way. Here we choose $L=\{1, 3, 6\}$, $r \in \{3, 6, 10\}$ and $M^{*} \in \{1, 2, \dots, 10\}$. To solve the least squares problem in \eqref{se2eq5}, we use the quasi-Newton method of the function \textit{fminunc} in MATLAB to approximate the solution. \newline
The results of our estimate are compared to other conventional estimates. In particular we compare the sparsely connected neural network estimate (abbr. \textit{neural-sc}) to a fully connected neural network (abbr. \textit{neural-fc}) with adaptively chosen number of hidden layers and number of neurons per layer. The selected values of these two parameters to be tested were $\{1, 2, 4, 6, 8, 10, 12\}$ for $L$ and $\{1, 2, \dots, 6, 8, 10\}$ for $r$. Beside this, we compare our neural network estimate to another sparsely connected neural network estimate, namely the network \textit{neural-x} defined in \cite{BK17}. The parameters $l,K,d^*, M^*$ of this estimate are chosen in a data-dependent way as described in \cite{BK17}. For instance, we select these parameters out of the set $\{0,1,2\}$ for $l$, out of $\{1, \dots, 5\}$ for $K$, out of $\{1, \dots, d^*\}$ for $d^*$, and out of $\{1, \dots, 5, 6, 11, 16, 21, \dots, 46\}$ for $M^*$.\\
Furthermore, we consider a nearest neighbor estimate (abbr. \textit{neighbor}). This means that the function value at a given point $x$ is approximated by the average of the values $Y_1, \dots, Y_{k_n}$ observed for the data points $X_1, \dots, X_{k_n}$, which are closest to $x$ with respect to the Euclidean norm (breaking the ties by indices). Here the parameter $k_n \in \N$ denoting the involved neighbors is chosen adaptively from the set $\{1, 2, 3\} \cup \{4, 8, 12, 16, \dots, 4 \cdot \lceil \frac{n_{train}}{4}\rceil\}$. Another competitive approach is the interpolation with radial basis function (abbr. \textit{RBF}). Here we use Wendland's compactly supported radial basis function $\phi(r) = (1-r)^6_+ \cdot (35r^2+18r+3)$, which can be found in \cite{LM02}. The radius $r$ that scales the basis functions is also selected adaptively from the set $\{0.1, 0.5, 1, 5, 30, 60, 100\}$. The last competitive approach is of course MARS. Here we used the ARESLab MATLAB toolbox provided by \cite{J16}. 
\newline
The $n$ observations (for $n \in \{100, 200\}$) $(X,Y), (X_1, Y_1), (X_2, Y_2), \dots, (X_n, Y_n)$ are chosen as independent and identically distributed random vectors with $X$ uniformly distributed on $[0,1]^{10}$ (in particular, the dimension of $X$ is $d=10$) and $Y$ generated by 
\begin{equation*}
Y = m_i(X) + \sigma_j \cdot \lambda_i \cdot \epsilon \ (i \in \{1,2,3\}, j \in \{1,2\})
\end{equation*}
for $\sigma_j \geq 0$, $\lambda_i \geq 0$ and $\epsilon$ standard normally distributed and independent of $X$. The $\lambda_i$ is chosen in way that respects the range covered by $m_i$ on the distribution of $X$. Since our regression functions perform differently on different polytopes we determine the interquartile range of $10^5$ realizations of $m_i(X)$ (additionally stabilized by taking the median of hundred repetitions of this procedure) not for the whole regression function, but on each set seperately and use the average of those values. For the regression functions below we got
% $\lambda_1=5.24$, 
$\lambda_1=2.72$, $\lambda_2=6.28$ and $\lambda_3=12.2$. The parameters scaling the noise are chosen as $\sigma_1 = 5\%$ and $\sigma_2 = 20\%$. 
\newline
\newline
The regression functions which were used to compare the different approaches are listed below.
\begin{align*}
%m_1(x) =  & \cot\left(\frac{\pi}{1+e^{x_1^2 + 2 \cdot x_2 + \sin(6 \cdot x_4^3)-3}}\right) \cdot 1_{H_1}(x)\\
%& +  \left(\cot\left(\frac{\pi}{1+e^{x_1^2 + 2 \cdot x_2 + \sin(6 \cdot x_4^3)-3}}\right) + e^{3 \cdot x_3 + 2 \cdot x_4}\right) \cdot 1_{\R^{10} \textbackslash H_1}(x)\\
%m_2(x) = & \cot\left(\frac{\pi}{1+e^{x_1^2 + 2 \cdot x_2 + \sin(6 \cdot x_4^3)-3}}\right) \cdot 1_{H_1}(x)\\
%& +  \left(\cot\left(\frac{\pi}{1+e^{x_1^2 + 2 \cdot x_2 + \sin(6 \cdot x_4^3)-3}}\right) + e^{3 \cdot x_3 + 2 \cdot x_4 -5x_5+ \sqrt{x_6+0.9x_7+0.1}}\right) \cdot 1_{\R^{10} \textbackslash H_1}(x)\\
%m_1(x) = &\left(\frac{2}{x_1+0.008} + 3 \cdot \log(x_2^7+0.1) \cdot x_4\right) \cdot 1_{H_1}(x)\\
%&+ \left(\frac{2}{x_1+0.008} + 3 \cdot \log(x_3+0.1) \cdot x_4\right) \cdot 1_{\R^{10} \textbackslash H_1}(x),\\
m_1(x) = & \left(\frac{10}{1+x_1^2}+5 \cdot \sin(x_3 \cdot x_4)+2\cdot x_5\right) \cdot 1_{H_1}(x)\\ 
& \quad + \left(\exp(x_1)+x_2^2+\sin(x_3 \cdot x_4) -3\right) \cdot 1_{\R^{10}\textbackslash H_1}(x),
\end{align*}
\begin{align*}
m_2(x) = &\left(\cot\left(\frac{\pi}{1+\exp(x_1^2+2 \cdot x_2+ \sin(6 \cdot x_4^3) -3)}\right)\right) \cdot 1_{H_1}(x)\\
\quad & + \left(\cot\left(\frac{\pi}{1+\exp(x_1^2+2 \cdot x_2+ \sin(6 \cdot x_4^3) -3)}\right)\right.\\
& \quad  + \exp\left(3 \cdot x_3+2 \cdot x_4- 5 \cdot x_1+ \sqrt{x_3 + 0.9 \cdot x_4+0.1}\right)\bigg) \cdot 1_{\R^{10} \textbackslash H_1}(x)\\
m_3(x) = &\left(2 \cdot \log(x_1 \cdot x_2+4 \cdot x_3 + |\tan(x_4)|\right) \cdot 1_{H_2 \cup H_3}(x) + \left(x_3^4 \cdot x_5^2 \cdot x_6 - x_4 \cdot x_7\right) \cdot 1_{H_2^C \cup H_3}(x)\\
\quad & + \left(3 \cdot x_8^2 + x_9 +2\right)^{0.1 +4 \cdot x_{10}^2} \cdot 1_{H_3^C}(x)
\end{align*}
with
\begin{align*}
%H_1 &= \{x \in \R^d: 0.1 \cdot x_1 + 0.4 \cdot x_2+0.3 \cdot x_3+0.1 \cdot x_4 \leq 0.45\}\\
%H_2 &= \{x \in \R^d: 4 \cdot x_1 + 7 \cdot x_4 \leq 4.5\}
H_1 = &\{x \in \R^{10}: 0.1 \cdot x_1 + 0.4 \cdot x_2 + 0.3 \cdot x_3 + 0.1 \cdot x_4+0.2 \cdot x_5 + \\
& \quad 0.3 \cdot x_6 + 0.6 \cdot x_7 + 0.02 \cdot x_8 + 0.7 \cdot x_9 + 0.6 \cdot x_{10} \leq 1.63\}\\
H_2 = &\{x \in \R^{10}: 0.1 \cdot x_1 + 0.4 \cdot x_2 + 0.3 \cdot x_3 + 0.1 \cdot x_4+0.2 \cdot x_5 + \\
& \quad 0.3 \cdot x_6 + 0.6 \cdot x_7 + 0.02 \cdot x_8 + 0.7 \cdot x_9 + 0.6 \cdot x_{10} \leq 1.6\}\\
H_3 = &\{x \in \R^{10}: 4 \cdot x_1 + 2 \cdot x_2 + x_3+4 \cdot x_4+x_5+x_6 \leq 7.5\}.
%H_4 = &\{x \in \R^{10}: 0.1 \cdot x_1 + 0.4 \cdot x_2 + 0.3 \cdot x_3 + 0.1 \cdot x_4+0.2 \cdot x_5 + \\
%& \quad 0.3 \cdot x_6 + 0.6 \cdot x_7 + 0.02 \cdot x_8 + 0.7 \cdot x_9 + 0.6 \cdot x_{10} \leq 1.8\}\\
\end{align*}

%\begin{center}
%\begin{tabular}{rrrrr}
%\hline
%$\lambda_1$ & $\lambda_2$ & $\lambda_3$ & $\lambda_4$ & $\lambda_5$\\
%\hline
%13.08 & 5.94 & 14 & 2.72 & 6.28\\
%\hline
%\end{tabular}
%\end{center}

%\begin{center}
%\begin{tabular}{rrr}
%\hline
%$\lambda_1$ & $\lambda_2$ & $\lambda_3$ \\
%\hline
% 14 & 2.72 & 6.28\\
%\hline
%\end{tabular}
%\end{center}
The quality of each of the estimates is determined by the empirical $L_2$-error, i.e. we calculate
\begin{equation*}
\epsilon_{L_2, N}(m_{n,i}) = \frac{1}{N} \sum_{k=1}^N \left(m_{n,i}(X_{n+k}) - m_i(X_{n+k}))^2\right),
\end{equation*}
where $m_{n,i}$ $(i=1, \dots, 4)$ is one of our estimates based on the $n$ observations and $m_i$ is our regression function. The input vectors $X_{n+1}, X_{n+2}, \dots, X_{n+N}$ are newly generated independent realizations of the random variable $X$, i.e. different from the $n$ input vectors for the estimate. We choose $N = 10^5$. We normalize our error by the error of the simplest estimate of $m_i$, i.e. the error of a constant function, calculated by the average of the observed data. Thus the errors given in our tables below are normalized error measures of the form $\epsilon_{L_2,N}(m_{n,i})/\bar{\epsilon}_{L_2,N}(avg)$. Here $\bar{\epsilon}_{L_2,N}(avg)$ is the median of $50$ independent realizations you obtain if you plug the average of $n$ observations into $\epsilon_{L_2,N}(\cdot)$. Since our simulation results depend on randomly chosen data points we repeat our estimation $50$ times by using differently generated random realizations of $X$ in each run. In Table \ref{tab1} and Table \ref{tab2} we listed the median (plus interquartile range IQR) of $\epsilon_{L_2,N}(m_{n,i})/\bar{\epsilon}_{L_2,N}(avg)$. 

%\begin{table}
%\caption{Median of the normalized empirical $L_2$-error for each estimate and regression functions $m_1$}
%\label{tab1}
%\begin{center}
%\small{
%\begin{tabular}{ccccc}
%\hline
%\multicolumn{5}{c}{$m_1$}\\
%\hline 
%\textit{noise} & \multicolumn{2}{c}{$5\%$} & \multicolumn{2}{c}{$20\%$} \\
%\hline
%\textit{sample size} & $n=100$ & $n=200$ & $n=100$ & $n=200$\\
%\hline
%$\bar{\epsilon}_{L_2,\bar{N}}(avg)$ & $407.0850$ & $404.1250$ & $407.4050$ & $406.7750$ \\
%\hline
%\textit{neural-sc} & $\mathbf{0.4648}(0.4322)$ & $0.2806 (0.2764)$ & $\mathbf{0.6604} (0.4849)$ & $\mathbf{0.3271} (0.5236)$ \\ 
%\textit{neural-fc} & $0.6693 (0.6826)$ & $\mathbf{0.1956} (0.3861)$  & $0.6856 (1.0252)$ & $0.3440 (0.6279)$ \\
%\textit{RBF} & $1.1108 (0.3138)$ & $1.0906 (0.2431)$ & $1.1081 (0.3205)$ & $1.0804 (0.3270)$\\
%\textit{neighbor} & $0.9023 (0.0746)$ & $0.8579 (0.1189$ & $0.9084 (0.1143)$ & $0.8654 (0.0894)$ \\
%\textit{MARS} & $1.5141 (5.3162)$ & $1.1571 (20.3491)$ & $1.3267 (7.4327)$ & $0.7124 (10.5108)$\\
%\hline
%\end{tabular}}
%\end{center}
%\end{table}

\begin{table}
\caption{Median of the normalized empirical $L_2$-error for each estimate and regression functions $ m_1, m_2$}
\label{tab1}
\begin{center}
\small{
\begin{tabular}{ccccc}
\hline
\multicolumn{5}{c}{$m_1$}\\
\hline 
\textit{noise} & \multicolumn{2}{c}{$5\%$} & \multicolumn{2}{c}{$20\%$} \\
\hline
\textit{sample size} & $n=100$ & $n=200$ & $n=100$ & $n=200$\\
\hline
$\bar{\epsilon}_{L_2,\bar{N}}(avg)$ &$29.5445$& $29.4330$&$29.4970$& $29.4375$\\
\hline
\textit{neural-sc} & $\mathbf{0.3809}(0.1902)$ & $\mathbf{0.1926}(0.1568)$ & $0.5113(0.3604)$ & $0.2971 (0.2546)$\\ 
\textit{neural-x} & $0.4412 (0.2653)$ & $0.2035 (0.2178)$ & $\mathbf{0.4674} (0.4427)$ & $\mathbf{0.2218} (0.3167)$\\
\textit{neural-fc} &$0.5040(0.3988)$& $0.2220(0.1568)$   & $0.4958(0.4742)$& $0.3016 (0.1928)$ \\
\textit{RBF} & $0.6856(0.1205)$& $0.6064(0.0670)$ & $0.7044(0.1150)$& $0.6173 (0.0754)$\\
\textit{neighbor} &$0.6387(0.0785)$& $0.5610(0.0489)$ & $0.6411 (0.0776)$& $0.5589 (0.0500)$ \\
\textit{MARS} &$0.6747(0.1433)$& $0.5091(0.0567)$ & $0.6949(0.1787)$& $0.5149 (0.0519)$\\
\hline
\end{tabular}}
\end{center}
\begin{center}
\small{
\begin{tabular}{ccccc}
\hline
\multicolumn{5}{c}{$m_2$}\\
\hline 
\textit{noise} & \multicolumn{2}{c}{$5\%$} & \multicolumn{2}{c}{$20\%$} \\
\hline
\textit{sample size} & $n=100$ & $n=200$ & $n=100$ & $n=200$\\
\hline
$\bar{\epsilon}_{L_2,\bar{N}}(avg)$ & $671.83$ & $670.77$ &$669.82$& $672.04$ \\
\hline
\textit{neural-sc} & $\mathbf{0.8108}(0.6736)$ & $\mathbf{0.5468}(0.6812)$ &$\mathbf{0.7453}(0.5348)$& $\mathbf{0.5146}(0.4298)$ \\ 
\textit{neural-x} & $0.8296 (0.3139)$ & $0.5543 (0.3884)$ & $0.8788 (0.5053)$ & $0.5488 (0.4127)$\\
\textit{neural-fc} &$1.0668(0.6779)$& $0.7792(0.4642)$ & $0.9678(0.4276)$& $0.8476(0.6150)$\\
\textit{RBF} &$1.0172(0.2613)$& $0.6896(0.3906)$ &$1.0179(0.2517)$& $0.6582 (0.3297)$\\
\textit{neighbor} &$0.8640(0.1086)$& $0.7990(0.1476)$ &$0.8657(0.0884)$& $0.7469(0.1156)$ \\
\textit{MARS} &$1.6299(1.5082)$& $3.4815(16.9055)$ &$1.6363(2.4886)$& $2.3530(10.0750)$\\
\hline
\end{tabular}}
\end{center}
\end{table}
\begin{table}[h]
\caption{Median of the normalized empirical $L_2$-error for each estimate and regression functions $m_3$}
\label{tab2}
\begin{center}
\small{
\begin{tabular}{ccccc}
\hline
\multicolumn{5}{c}{$m_3$}\\
\hline 
\textit{noise} & \multicolumn{2}{c}{$5\%$} & \multicolumn{2}{c}{$20\%$} \\
\hline
\textit{sample size} & $n=100$ & $n=200$ & $n=100$ & $n=200$\\
\hline
$\bar{\epsilon}_{L_2,\bar{N}}(avg)$ & $9023.9$ & $9018.4$ & $9117.1$ & $9017.4$ \\
\hline
\textit{neural-sc} & $0.5983(0.6832)$ & $\mathbf{0.2006} (0.3523)$ & $\mathbf{0.5521} (0.3977)$ & $ 0.3223 (0.3143)$ \\ 
\textit{neural-x} &  $\mathbf{0.5168} (0.6809)$ & $0.3156 (0.2091)$ & $0.5555 (0.6642)$ & $\mathbf{0.3147} (0.2386)$\\
\textit{neural-fc} & $0.7337 (0.6276)$ & $0.3657 (0.4543)$  & $0.8311 (0.4058)$ & $0.3397 (0.4208)$ \\
\textit{RBF} & $0.6764 (0.4601)$ & $0.5527 (0.3601)$ & $0.6580 (0.4698)$ & $0.5312 (0.3780)$\\
\textit{neighbor} & $0.8188 (0.1170)$ & $0.7137 (0.0985)$ & $0.8024 (0.1117)$ & $0.7191 (0.0987)$ \\
\textit{MARS} & $0.9925 (1.7966)$ & $0.6596 (0.7020)$ & $1.1440 (5.5270)$ & $0.6445 (0.7419)$\\
\hline
\end{tabular}}
\end{center}
\end{table}
We observe that our estimate outperforms the other approaches in 8 of 12 examples for regression functions with low local dimensionality. Especially in cases $m_1$ and $m_3$, the error of our estimate is about half the error in each of the other approaches for $n=200$ and $\sigma=0.05$, except for the error of the other neural networks. We also observe, that the relative improvement of our estimate (and of the other networks) with an increasing sample size is much larger than the improvement for most of the other approaches (except in $m_2$ for the RBF and in $m_3$ for MARS). 
%Only for $m_1$ and $\sigma=0.2$ the improvement of MARS is better 
This could be a plausible indicator for a better rate of convergence. \\
It makes sense that we also get good approximations for the fully connected neural networks, since some of the sparse networks can be expressed by fully connected ones (e.g., choosing some weights as zero). The estimate \textit{neural-x} of \cite{BK17} was originally constructed to estimate regression functions with some composition assumption, for instance $(p,C)$-smooth generalized hierarchical interaction models. Since our regression functions follow a $(p,C)$-smooth generalized hierarchical interaction model on each polytope, it is plausible that this estimate also performs well for those regression functions. Nevertheless, with regard to our simulation results we see, that (with four exception) our sparse neural networks perform better than the other neural network estimates. 

\section{Real-world data experiment}
The different approaches of the simulation study were further tested on a real--world data set to emphasize the practical relevance of our estimate. The data set under study was the earlier mentioned 2--year usage log of a bike sharing system named Captial Bike Sharing (CBS) at Washington, D.C., USA (\cite{FG13}), where we conjecture some low local dimensionality in the data set, which fits our assumption on the regression function. 
 The data set consists of $17379$ data points, where each of them represents one hour of a day between 2011 and 2012; $500$ were used for training and testing and the rest is used to compute the errors contained in Table \ref{tab4}. We used the same parameter sets as in the simulation study for all of our estimates and normalized the results again with the simplest estimate i.e. the average of the observed data . Table \ref{tab4} summarizes the results.
\begin{table}[h]
\caption{Normalized empirical $L_2$-risk for each estimate for the bike sharing data}
\label{tab4}
\begin{center}
\small{
\begin{tabular}{cccccc}
\hline
\textit{neural-sc} & \textit{neural-x} & \textit{neural-fc} & \textit{RBF} & \textit{neighbor} & \textit{MARS}\\
\hline
$\textbf{0.1680}$ & $0.3706$ & $0.5924$ & $0.8121$ & $0.6829$ & $0.3970$\\
\hline
\end{tabular}}
\end{center}
\end{table}
Again we observe that our estimate outperforms the others i.e. the error of our estimate is about half the error of the second best approach (MARS). Hence our assumption of low local dimensionality seems plausible, at least for this real data set, since the estimate following this assumption outperforms all other estimates. 
 
%
%\begin{figure}
%\begin{center}
%%\includegraphics[width=3in]{fig1.pdf}
%\end{center}
%\caption{Consistency comparison in fitting surrogate model in the tidal
%power example. \label{fig:first}}
%\end{figure}

%\begin{table}
%\caption{D-optimality values for design $X$ under five different scenarios.  \label{tab:tabone}}
%\begin{center}
%\begin{tabular}{rrrrr}
%one & two & three & four & five\\\hline
%1.23 & 3.45 & 5.00 & 1.21 & 3.41 \\
%1.23 & 3.45 & 5.00 & 1.21 & 3.42 \\
%1.23 & 3.45 & 5.00 & 1.21 & 3.43 \\
%\end{tabular}
%\end{center}
%\end{table}

%\bigskip
%\newpage
%\begin{center}
%{\large\bf SUPPLEMENTARY MATERIAL}
%\end{center}
%Supplement description: Proof of the main result
%\begin{description}
%
%\item[Section A:] Proof of the main result.
%\item[Section B:] The algorithm used for the sparsely connected neural networks in the Simulation Studies.

%\item[R-package for  MYNEW routine:] R-package ÒMYNEWÓ containing code to perform the diagnostic methods described in the article. The package also contains all datasets used as examples in the article. (GNU zipped tar file)
%
%\item[HIV data set:] Data set used in the illustration of MYNEW method in Section~ 3.2. (.txt file)

%\end{description}

\bibliographystyle{apalike}
\bibliography{Literatur}

%\section{BibTeX}
%
%We hope you've chosen to use BibTeX!\ If you have, please feel free to use the package natbib with any bibliography style you're comfortable with. The .bst file agsm has been included here for your convenience. 
%
%\bibliographystyle{Chicago}
%
%\bibliography{Bibliography-MM-MC}
\newpage 
\renewcommand*\thesection{\Alph{section}}
\setcounter{section}{0}
  \begin{center}
    {\LARGE\bf Supplementary material to "Estimation of a function of low local dimensionality by deep neural networks"}
\end{center}
\noindent
This supplement contains auxiliary results and a detailed proof of Theorem 1 in \linebreak Section \ref{app4}. 
\section{An error bound for sparse neural network regression estimates}
\label{se5}
The proof of our main result is based on a result that analyzes the connection between DNNs and multivariate adpative regression splines (MARS). We show that our estimates satisfy an oracle inequality which implies that
    (up to a logarithmic factor) the error of our estimates is
    at least as small as the optimal possible error bound which one would expect for MARS 
    in case that this procedure would work in the optimal way. Before we show this, we shortly introduce the procedure of MARS.
 \subsection{MARS}
\cite{Fr91} introduced
a procedure called MARS, which uses
 a hierarchical forward/backward
stepwise subset selection procedure for choosing a subbasis from
a (complete) linear truncated power tensor product basis.
Here the subbasis is chosen in a data--dependent way
from the basis $\B$ consisting of all functions of the form
\begin{equation}
\label{se1eq4}
B(x)= \prod_{j \in J}
( s_j \cdot (x^{(j)} - a_j))_+
\end{equation}
(with the convention $\prod_{j \in \emptyset} z_j=1$),
where $z_+=\max\{z,0\}$ and $J \subseteq \{1, \dots, d\}$,
$s_j \in \{-1,1\}$ and $a_j \in \R$ are parameters of the above
basis functions. In order to reduce the complexity of the procedure
the locations of $a_j$ are restricted to the values of the $j$--th
component of $x$--values of the given data. As soon as such a
subbasis $B_1$, \dots, $B_K$ (where $K \in \N$ is the number of basis
functions, which is also data dependent) are chosen, the principle
of least squares is used to construct an estimate of $m$ by
\[
m_n(x) = \sum_{k=1}^K \hat{a}_k \cdot B_k(x)
\]
where
\[
(\hat{a}_k)_{k=1,\dots, K}
=
\arg \min_{(a_k)_{k=1,\dots, K} \in \R^K}
\frac{1}{n}
\sum_{i=1}^n
| Y_i - \sum_{k=1}^K a_k \cdot B_k(X_i)|^2.
\]
For any fixed basis $B_1$, \dots, $B_K$ the expected $L_2$
error of the above estimate satisfies basically a bound
of the form
\begin{eqnarray*}
  &&
  \EXP \int |m_n(x)-m(x)|^2 \PROB_X(dx)
  \\
  &&
\leq
const \cdot \frac{K}{n}
+
\min_{(a_k)_{k=1,\dots, K} \in \R^K}
\int
|\sum_{k=1}^K a_k \cdot B_k(x)-m(x)|^2 \PROB_X(dx)
\end{eqnarray*}
(cf., e.g., Theorem 11.1 and Theorem 11.3 in \cite{GKKW02}).
So if we have an oracle which produces the optimal subset of
basis functions, the corresponding least squares estimate
would satisfy
\begin{eqnarray}
  &&
  \EXP \int |m_n(x)-m(x)|^2 \PROB_X(dx)
\leq
\inf_{K \in \N, B_1, \dots, B_K \in \B}
\Bigg(
const \cdot \frac{K}{n}
  \nonumber
  \\
  &&
\hspace*{3cm}
+
\min_{(a_k)_{k=1,\dots, K} \in \R^K}
\int
|\sum_{k=1}^K a_k \cdot B_k(x)-m(x)|^2 \PROB_X(dx)
\Bigg).
\label{se1eq2}
\end{eqnarray}
The great advantage of such an error bound is that it has the
power of exploiting low local dimensionality of the regression function.
For instance, in case that the multivariate regression function depends
globally on $d$ variables, but depends in any local region only
on $d^* \in \{1, \dots,d\}$ of them,
we should be able to derive from the above bound
a rate of convergence depending only on $d^*$ and not on $d$.
\\
\\
The aim of MARS is to use the data to produce  an optimal subbasis
with a hierarchical forward/backward stepwise subset selection procedure.
Of course, there is no guarantee that the resulting basis is as good
as the basis produced by an oracle, therefore (\ref{se1eq2}) does
not hold for MARS.
\subsection{Deep Learning and MARS: A Connection}
In the following we show that our sparse neural network estimates achieve the error bound \eqref{se1eq2}. The result will be proven for some generalization $\B_{n,M,K_1}^*$
of the basis $\B$. Therefore we introduce polynomial splines i.e., sets of piecewiese polynomials 
satisfying a global smoothness condition, and a corresponding B-spline basis consisting of basis functions 
with compact support as follows: 

\begin{definition}\label{de5}
%\begin{itemize}
Let $K \in \N$ and $M \in \N_0$.
\\
Choose $t_j \in \R$ $(j \in \{-M, \dots, K+M\})$, such that $t_{-M} < t_{-M+1} < \dots < t_{K+M}$ and set $t = \{t_j\}_{j=-M, \dots, K+M}$. For $j \in \{-M, -M+1, \dots, K-1\}$ let $B_{j,M,t}: \R \to \R$ be the univariate B-Spline of degree $M$ recursively defined by
\begin{enumerate}
\item[(i)] 
\begin{equation*}
B_{j,0,t} = 
\begin{cases}
1, \quad x \in [t_j, t_{j+1})\\
0, \quad x \notin [t_j, t_{j+1})
\end{cases}
\end{equation*}
for $j \in \{-M, \dots, K+M-1\}$ and
\item[(ii)]
\begin{equation}\label{le4eq3}
B_{j, l+1,t} (x) = \frac{x-t_j}{t_{j+l+1}-t_j} B_{j,l,t}(x) + \frac{t_{j+l+2}-x}{t_{j+l+2}-t_{j+1}} B_{j+1,l,t}(x)
\end{equation}
for $j \in \{-M, \dots, K+M-l-2\}$ and $l \in \{0, \dots, M-1\}$.
\end{enumerate}
%\textbf{b)}
%For $\mathbf{j} = (j_1, \dots, j_d) \in \Z^d$ we define the tensor product B-Spline $B_{\mathbf{j},M}: \R^d \to \R$ by 
%\begin{equation*}
%B_{\mathbf{j},M}(x^{(1)}, \dots, x^{(d)}) = B_{j_1,M}(x^{(1)}) \cdots B_{j_d,M}(x^{(d)}) \quad (x^{(1)}, \dots, x^{(d)} \in \R).
%\end{equation*}
%\end{itemize}
\end{definition}

\noindent
Choose $M, K_1, d, n \in \N$ and $c_{11}, c_{12} > 0$.
The result will be shown for some generalization $\B_{n,M,K_1}^*$ of $\B$, which consists of all functions of the form
\begin{equation}\label{Bspline_gen}
B(x) = \prod_{v \in J_1} B_{j_{v},M, t_v}(x^{(i_v)}) \cdot \prod_{k \in J_2} \left(\sum_{j=1}^d \alpha_{k,j} \cdot (x^{(i_j)} - \gamma_{k,j})\right)_+
\end{equation}
where $J_1 \subseteq \left\{1, \dots, d\right\}$, $J_2 \subseteq \left\{1, \dots, K_1\right\}$, $K \in \N$, $j_{v} \in \{-M, -M+1, \dots, K-1\}$, \linebreak $i_v, i_j \in \{1, \dots, d\}$, $t_v = \{t_{v,k}\}_{k=-M, \dots, K+M}$ with $t_{v,k}, \alpha_{k,j}, \gamma_{k,j} \in [-c_{11} \cdot n^{c_{12}}, c_{11} \cdot n^{c_{12}}]$ and $t_{v,k+1} - t_{v,k} \geq \frac{1}{n}$. 
% -----------------------------------Netzwerk anpassen------------------------------------
\begin{theorem}
  \label{th1}
Let $\beta_n = c_{3} \cdot \log(n)$ for some constant $c_{3}>0$. Assume that the distribution of $(X,Y)$ satisfies
\begin{align}\label{subgaus}
\mathbf E \left(\exp({c_{4}\cdot |Y |^2})\right) <\infty
\end{align}
for some constant $c_{4}>0$, that $supp(X) \subseteq [-A,A]^d$ for some $A \geq 1$ and that the regression function $m$ is
bounded in absolute value. Let $M, K_1 \in \N$ and let the least squares neural network
regression estimate $m_n$ be defined as in Section 2
with parameters 
\[
L=3K_1+d \cdot (M+2)-1, \quad  r=2^{M-1} \cdot 16+\sum_{k=2}^M 2^{M-k+1}+d+5, \quad \alpha_n= c_1 \cdot n^{c_2}
\]
  and $n_l=\lceil  n / 2 \rceil$.  Assume
 that the sigmoidal function $\sigma$ is $2$--admissible,
 and that
 $c_1$, $c_2$, $c_{10}>0$ are suitably large.
Then we have for any $n > 7$:
\begin{eqnarray*}
  &&
  \EXP \int |m_n(x)-m(x)|^2 \PROB_X(dx)
 \leq
  (\log n)^3 \cdot
\inf_{\substack{I \in \N, \ B_1, \dots, B_I \in \mathcal{B}_{n,M,K_1}^{*}}}
\Bigg(
c_9 \cdot \frac{I}{n}
  \\
  &&
\hspace*{3cm}
+
\min_{(a_i)_{i=1, \dots, I}\in [-c_{10} \cdot n, c_{10} \cdot n]^I}
\int
|\sum_{i=1}^I a_{i} \cdot B_{i}(x)-m(x)|^2 \PROB_X(dx)
\Bigg).
\end{eqnarray*}
  \end{theorem}
\noindent
  
\noindent
\begin{remark}
By combining our proofs with the techniques
introduced in \cite{Sch17} it is possible to show
that a similar result also holds for neural networks using
the ReLU-function $\sigma_{ReLU}(x)=\max\{0,x\}$ 
as activation function.
\end{remark}

\noindent 
Corollary \ref{cor1} concerns the result of Theorem \ref{th1}, if we choose $J_1 = \emptyset$ in \eqref{Bspline_gen}. This results in a basis $\B_{n,K_1}^*$, which consists of all functions of the form
\begin{equation}\label{basis_trunc}
B(x) = \prod_{k \in J_2} \left(\sum_{j=1}^d \alpha_{k,j} \cdot (x^{(i_j)} - \gamma_{k,j})\right)_+,
\end{equation}
with indices defined as in \eqref{Bspline_gen}. If we choose here $K_1 = d$, $\alpha_{k,k} \in \{-1,1\}$, $\alpha_{k,j}=0$ for $j \neq k$ and $i_j = j$,  the basis $\B$ is contained in $\B_{n,K_1}^*$, provided the location of $a_j$ in $\B$ are restricted to the values of the $j$-th component of the $x$-value.

\begin{corollary}
  \label{cor1}
Assume that the conditions of Theorem \ref{th1} are satisfied. Let the least squares neural network
regression estimate $m_n$ be defined as in Section 2
with parameters \newline
 $L=3\cdot (K_1+d)-1$, $r=d+21$, $\alpha_n= c_1 \cdot n^{c_2}$  and $n_l=\lceil  n / 2 \rceil$.
Then we have for any $n>7$:
\begin{eqnarray*}
  &&
  \EXP \int |m_n(x)-m(x)|^2 \PROB_X(dx)
 \leq
  (\log n)^3 \cdot
\inf_{K \in \N, B_1, \dots, B_K \in \B_{n,K_1}^*}
\Bigg(
c_{13} \cdot \frac{K}{n}
  \\
  &&
\hspace*{3cm}
+
\min_{(a_k)_{k=1,\dots, K} \in [-c_{10} \cdot n, c_{10} \cdot n]^K}
\int
|\sum_{k=1}^K a_k \cdot B_k(x)-m(x)|^2 \PROB_X(dx)
\Bigg).
\end{eqnarray*}
  \end{corollary}
  \begin{proof}
  By choosing $M=1$ and $\mathcal{B}_{n,K_1}^*$ instead of $\mathcal{B}_{n,M,K_1}^*$ this follows directly by application of Theorem \ref{th1}.
  \end{proof}

\section{Approximation properties of neural networks}
In this section we present approximation properties of neural networks, which are needed to prove Theorem \ref{th1}. 
\begin{lemma}
  \label{le1}
Let $\sigma:\R \rightarrow \R$ be a function, let $R \geq 1$ and $a>0$.

\noindent
{\bf a)}
Assume that $\sigma$ is two times continuously differentiable and
let
$t_{\sigma,id} \in \R$ be such that $\sigma^\prime(t_{\sigma,id}) \neq
0$.
Then
\[
f_{id}(x)
=
\frac{R}{\sigma^\prime(t_{\sigma,id})}
\cdot
\left(
\sigma \left(
\frac{x}{R} + t_{\sigma,id}
\right)
-
\sigma(t_{\sigma,id})
\right)
\in \F(1,1,c_{14} \cdot R)
\]
satisfies for any $x \in [-a,a]$:
\[
| f_{id}(x)-x|
\leq
\frac{
\| \sigma^{\prime \prime}\|_{\infty} \cdot a^2
}{
2 \cdot |\sigma^\prime(t_{\sigma,id})|
}
\cdot
\frac{1}{R}.
\]

\noindent
{\bf b)}
Assume that $\sigma$ is three times continuously differentiable and
let
$t_{\sigma,sq} \in \R$ be such that $\sigma^{\prime \prime}(t_{\sigma,sq}) \neq
0$.
Then
\[
f_{sq}(x)
=
\frac{R^2}{\sigma^{\prime \prime}(t_{\sigma,sq})}
\cdot
\left(
\sigma \left(
\frac{2x}{R} + t_{\sigma,sq}
\right)
-
2 \cdot
\sigma \left(
\frac{x}{R} + t_{\sigma,sq}
\right)
+
\sigma(t_{\sigma,sq})
\right)
\in \F(1,2,c_{15} \cdot R^2)
\]
satisfies for any $x \in [-a,a]$:
\[
| f_{sq}(x)-x^2|
\leq
\frac{5 \cdot 
\| \sigma^{\prime \prime \prime}\|_{\infty} \cdot a^3
}{3 \cdot |\sigma^{\prime \prime}(t_{\sigma,sq})|
}
\cdot
\frac{1}{R}.
\]
\end{lemma}

\noindent
\begin{proof}
   The result follows in a straightforward way from the proof of
   Theorem 2 in \cite{ScTs98}. For the sake of
   completeness we provide nevertheless the detailed proof below.

We get by Taylor expansion of order $2$
\begin{eqnarray*}
&&
| f_{id}(x)-x|
\\
&&
=
\left|
\frac{R}{\sigma^\prime(t_{\sigma,id})}
\cdot
\left(
\sigma \left(
t_{\sigma,id}
\right)
+
\sigma^\prime(t_{\sigma,id})
\frac{x}{R}
+
\frac{1}{2}
\sigma^{\prime \prime}(\xi)
\frac{x^2}{R^2}
-
\sigma(t_{\sigma,id})
\right)
-
x
\right|
\\
&&
=
\left|
\frac{R}{\sigma^\prime(t_{\sigma,id})}
\cdot
\frac{1}{2}
\sigma^{\prime \prime}(\xi)
\frac{x^2}{R^2}
\right|
\leq
\frac{
\| \sigma^{\prime \prime}\|_{\infty} \cdot a^2
}{
2 \cdot |\sigma^\prime(t_{\sigma,id})|
}
\cdot
\frac{1}{R}.
\end{eqnarray*}
The second part follows in the same way by using twice
Taylor expansion of order $3$.
\end{proof}
In the sequel we will use the abbreviations
\begin{align*}
f_{id}(z) = \left(f_{id}\left(z^{(1)}\right), \dots, f_{id}\left(z^{(d)}\right)\right) = \left(z^{(1)}, \dots, z^{(d)}\right), \quad z \in \Rd
\end{align*}
and
\begin{align*}
&f_{id}^0(z) = z, \quad z \in \R^d\\
&f_{id}^{t+1}(z) = f_{id}\left(f_{id}^t(z)\right) = z, \quad t \in \N_0, z \in \R^d.
\end{align*}
\begin{lemma}
\label{le2}
Let $\sigma: \mathbb{R} \to [0,1]$ be 2-admissible according to
Definition 3. Then for any $R \geq 1$ and any $a >0$ the neural network
\begin{eqnarray*}
  f_{mult}(x,y) &=&
\frac{R^2}{4 \cdot \sigma^{\prime \prime}(t_\sigma)}
\cdot
\Bigg(
\sigma \left(
\frac{2 \cdot (x+y)}{R} + t_\sigma
\right)
-
2 \cdot
\sigma \left(
\frac{x+y}{R} + t_\sigma
\right)
\\
&&
\hspace*{2cm}
-
\sigma \left(
\frac{2 \cdot (x-y)}{R} + t_\sigma
\right)
+
2 \cdot
\sigma \left(
\frac{x-y}{R} + t_\sigma
\right)
\Bigg)\\
&&
\quad
\in \F(1,4,c_{16} \cdot R^2)
\end{eqnarray*}
satisfies for any $x,y \in [-a,a]$:
\begin{equation*}
|f_{mult}(x,y) - x \cdot y| \leq
\frac{20 \cdot \|
\sigma^{\prime \prime \prime}
\|_{\infty} \cdot a^3}{
3 \cdot |\sigma^{ \prime \prime} (t_\sigma)|}
 \cdot \frac{1}{R}.
\end{equation*}
\end{lemma}

\begin{proof}
Let $f_{sq}$ be the network of Lemma \ref{le1}
satisfying
\begin{equation*}
|f_{sq}(x) - x^2| \leq
\frac{
40 \cdot
\| \sigma^{\prime \prime \prime}\|_{\infty} \cdot a^3
}{3 \cdot |\sigma^{\prime \prime}(t_{\sigma})|
}
\cdot
\frac{1}{R}
\end{equation*}
for $x \in [-2a, 2a]$, and set
\[
f_{mult}(x,y)
=
\frac{1}{4} \cdot \left( f_{sq}(x+y)-f_{sq}(x-y) \right).
\]
Since
\begin{equation*}
x \cdot y = \frac{1}{4} \left((x+y)^2 - (x-y)^2\right)
\end{equation*}
we have
\begin{align*}
\begin{split}
|f_{mult}(x,y) - x \cdot y|
&\leq \frac{1}{4} \cdot \left|f_{sq}(x+y) - (x+y)^2\right| + \frac{1}{4} \cdot \left|(x-y)^2 - f_{sq}(x-y)\right|\\
&\leq
2 \cdot \frac{1}{4} \cdot
\frac{
40 \cdot
\| \sigma^{\prime \prime \prime}\|_{\infty} \cdot a^3
}{
3 \cdot |\sigma^{\prime \prime}(t_{\sigma})|
}
\cdot
\frac{1}{R}.
\end{split}
\end{align*}
for $x,y \in [-a, a]$.
\end{proof}

\begin{lemma}\label{le3}
  Let $\sigma: \mathbb{R} \to [0,1]$ be 2-admissible
  according to Definition 3.
Let $f_{mult}$ be the neural network from Lemma \ref{le2}
    and let $f_{id}$ be the network from Lemma \ref{le1}.
Assume
\begin{equation}
\label{le3eq1}
a \geq 1 \quad \mbox{and} \quad
R
\geq
\max \left(
\frac{
\| \sigma^{\prime \prime}\|_{\infty}\cdot a
}{
2 \cdot | \sigma^\prime (t_{\sigma.id})|
}, 1
\right).
\end{equation}
Then the neural network
\begin{eqnarray*}
  f_{ReLU}(x) &=&
f_{mult} \left(f_{id}(x),\sigma \left(R \cdot x \right) \right)
  \\
  &=&
  \sum_{k=1}^4 d_k \cdot \sigma \left(\sum_{i=1}^2 b_{k,i} \cdot \sigma(a_i \cdot x + t_{\sigma})
  +
  b_{k,3} \cdot \sigma (a_3 \cdot x)
  +  t_{\sigma} \right)
\end{eqnarray*}
satisfies
\begin{equation*}
|f_{ReLU}(x) - \max\{x,0\}| \leq56 \cdot \frac{
\max \left\{
\| \sigma^{\prime \prime}\|_{\infty},
\| \sigma^{\prime \prime \prime}\|_{\infty},1
\right\}
}{
\min \left\{
2 \cdot |\sigma^\prime (t_{\sigma.id})|,
|\sigma^{\prime \prime} (t_{\sigma})|, 1
\right\}
} \cdot a^3 \cdot \frac{1}{R}
\end{equation*}
for all $x \in [-a, a]$. Here the weights
$d_k$, $b_{k,i}$ and $t_\sigma$ of this neural network are bounded in
absolute value by
\[
\alpha=c_{17} \cdot R^2,
\]
and consequently this network is contained in $\F(2,4,c_{17} \cdot R^2)$.
\end{lemma}

\begin{proof}
Since $\sigma$ is admissible we have
\begin{equation}
  \label{ple3eq1a}
| \sigma (R \cdot x) - \mathds{1}_{[0,\infty)}(x)|
\leq
\frac{1}{R \cdot |x|}
\quad (x \in \R \setminus \{0\}).
\end{equation}
By Lemma \ref{le1} and Lemma \ref{le2}
we have
\begin{equation}
\label{ple3eq1}
| f_{id}(x)-x|
\leq
\frac{
\| \sigma^{\prime \prime}\|_{\infty} \cdot a^2
}{
2 \cdot |\sigma^\prime(t_{\sigma,id})|
}
\cdot
\frac{1}{R}
\quad
\mbox{for }
x \in [-a,a]
\end{equation}
and
\[
|f_{mult}(x,y) - x \cdot y| \leq
\frac{160 \cdot \|
\sigma^{\prime \prime \prime}
\|_{\infty} \cdot a^3}{
3 \cdot |\sigma^{ \prime \prime} (t_\sigma)|}
 \cdot \frac{1}{R}
\quad
\mbox{for }
x \in [-2a,2a].
\]

\noindent
By inequalities (\ref{le3eq1}) we can conclude that
$f_{id}(x)$ and $\sigma (R \cdot x)$
are both contained in $[-2a,2a]$.
Using this together with
(\ref{ple3eq1a})
and the above inequalities we can conclude
  \begin{eqnarray*}
    &&
    |f_{ReLU}(x) - \max\{x,0\}|
    \\
    &&
    =
    |  f_{mult}\left(f_{id}(x),\sigma\left(R \cdot x\right) \right) - x \cdot  \mathds{1}_{[0,\infty)}(x)|
      \\
      &&
      \leq
      |  f_{mult}\left(f_{id}(x),\sigma\left(R \cdot x\right) \right)
-
f_{id}(x) \cdot \sigma\left(R \cdot x\right)
|
\\
&&
\quad
+
|
f_{id}(x) \cdot \sigma\left(R \cdot x\right)
      - x \cdot \sigma\left(R \cdot x\right)|
+
|
x \cdot \sigma\left(R \cdot x\right)
- x \cdot  \mathds{1}_{[0,\infty)}(x)|
     \\
      &&
     \leq
\frac{160 \cdot \|
\sigma^{\prime \prime \prime}
\|_{\infty} \cdot a^3}{
3 \cdot |\sigma^{ \prime \prime} (t_\sigma)|}
 \cdot \frac{1}{R}
+
\frac{
\| \sigma^{\prime \prime}\|_{\infty} \cdot a^2
}{
2 \cdot |\sigma^\prime(t_{\sigma,id})|
}
\cdot
\frac{1}{R}
\cdot 1
+
\frac{1}{R}
\\
&&
\leq
56 \cdot \frac{
\max \left\{
\| \sigma^{\prime \prime}\|_{\infty},
\| \sigma^{\prime \prime \prime}\|_{\infty},1
\right\}
}{
\min \left\{
2 \cdot |\sigma^\prime (t_{\sigma.id})|,
|\sigma^{\prime \prime} (t_{\sigma})|, 1
\right\}
} \cdot a^3 \cdot \frac{1}{R}
    \end{eqnarray*}
for all $x \in [-a,a]$.
\end{proof}

\begin{lemma}
  \label{le4}
 Let $\sigma: \mathbb{R} \to [0,1]$ be 2-admissible
  according to Definition 3.
Let $f_{ReLU}$ be the neural network from Lemma \ref{le3}.
Let $a,b \geq 1$, $d \in \N$, $\gamma_1, \dots, \gamma_d \in [-a,a]$ and
$\alpha_1, \dots, \alpha_d \in [-b,b]$.
Assume
\begin{equation}
\label{le4eq1}
R
\geq
\max \left\{
\frac{
\| \sigma^{\prime \prime}\|_{\infty}\cdot d \cdot a \cdot b
}{
| \sigma^\prime (t_{\sigma.id})|
},1 \right\}.
\end{equation}
Then the neural network
\begin{eqnarray*}
  f_{trunc}(x) &=&  f_{ReLU}\left(\sum_{k=1}^d \alpha_k \cdot (x^{(k)}-\gamma_k)\right)
  \end{eqnarray*}
satisfies
\begin{equation*}
\left|f_{trunc}(x) - \max\left\{\sum_{k=1}^d \alpha_k \cdot (x^{(k)}-\gamma_k),0\right\}\right|
\leq
448
\cdot \frac{
\max \left\{
\| \sigma^{\prime \prime}\|_{\infty},
\| \sigma^{\prime \prime \prime}\|_{\infty},1
\right\}
}{
\min \left\{
2 \cdot |\sigma^\prime (t_{\sigma.id})|,
|\sigma^{\prime \prime} (t_{\sigma})|, 1
\right\}
} \cdot d^3 \cdot a^3 \cdot b^3 \cdot \frac{1}{R}
\end{equation*}
for all $x \in [-a, a]^d$. Here the weights of this neural network are bounded in
absolute value by
\[
\alpha= c_{18} \cdot R^2 \cdot \max \left\{
1, |\alpha_1|,\dots,|\alpha_d|,\left|\sum_{k=1}^d \alpha_k \cdot \gamma_k\right|
\right\}
\]
and consequently this network is contained in $\F(2,4,\alpha)$.
\end{lemma}

\begin{proof}
For $x \in [-a,a]^d$ we have
\[
\sum_{k=1}^d \alpha_k \cdot (x^{(k)}-\gamma_k)
\in
[-2 \cdot d \cdot a \cdot b, 2 \cdot d \cdot a \cdot b].
\]
Application of Lemma \ref{le3} with $a$ replaced by
$2 \cdot d \cdot a \cdot b$ yields the assertion. 
\end{proof}

%-----------------------Lemma 5-----------------------
\begin{lemma}\label{le5}
Let $\sigma: \R \to [0,1]$ be $2$-admissible according to Definition 3. Let $M, K \in \N$, $j \in \{-M,-M+1, \dots, K-1\}$ and $a \geq 1$.
Assume 
\begin{align*}
R
\geq \max&\left\{ M \cdot  \frac{9 \cdot \Vert \sigma^{''}\Vert_{\infty} \cdot (an)^2}{2 \cdot |\sigma^{'}(t_{\sigma.id})|},\right.\\
 &\left. \quad  (4 \cdot 3 \cdot (M-1))^{M-2} \cdot \left(an\right)^{M+1} \cdot 4 \cdot 448\cdot \frac{\max\{\Vert \sigma^{''} \Vert_{\infty}, \Vert \sigma^{'''} \Vert_{\infty},1\}}{\min\{2 \cdot |\sigma^{'} (t_{\sigma .id})|, |\sigma^{''}(t_{\sigma})|,1\}}\right\}.
\end{align*}
Let $f_{id}$ be the network from Lemma \ref{le1} , $f_{mult}$ be the network from Lemma \ref{le2} and $f_{ReLU}$ be the network from Lemma \ref{le3}.  Let $B_{j,M,t}: \R \to \R$ be a univariate B-Spline of degree $M$ according to Definition 4 with knot sequence $t = \{t_k\}_{k=-M, \dots, M+K}$ such that $t_k \in [-a,a]$ and $t_{k+1} - t_k \geq \frac{1}{n}$. 
Then the neural network $f_{B_{j,M,t}}$ 
 recursively defined by 
\begin{align*}
f_{B_{j,l+1,t}}(x) = &f_{mult}\left(f_{id}^{l+1}\left(\frac{x-t_j}{t_{j+l+1}-t_j}\right), f_{B_{j,l,t}}(x)\right)\\
&+f_{mult}\left(f_{id}^{l+1}\left(\frac{t_{j+l+2} - x}{t_{j+l+2}-t_{j+1}}\right), f_{B_{j+1,l,t}}(x)\right)
\end{align*}
with $l=1, \dots, M-1$ and
\begin{align*}
f_{B_{j,1,t}}(x) = &f_{ReLU}\left(\frac{x-t_j}{t_{j+1}-t_j}\right) - f_{ReLU}\left(\frac{x-t_{j+1}}{t_{j+1}-t_j}\right)\\
&- f_{ReLU}\left(\frac{x-t_{j+1}}{t_{j+2}-t_{j+1}}\right) + f_{ReLU}\left(\frac{x-t_{j+2}}{t_{j+2}-t_{j+1}}\right)
\end{align*}
satisfies
\begin{align*}
\left|f_{B_{j,M,t}}(x) - B_{j,M,t}(x)\right| \leq (4 \cdot 3 \cdot M)^{M-1} \cdot \left(an\right)^{M+2}\cdot  4 \cdot 448 \cdot \frac{\max\{\Vert \sigma^{''} \Vert_{\infty}, \Vert \sigma^{'''} \Vert_{\infty},1\}}{\min\{2 \cdot |\sigma^{'}(t_{\sigma.id})|, |\sigma^{''}(t_{\sigma})|,1\}} \cdot \frac{1}{R}
\end{align*}
for all $x \in [-a,a]$.
Here the weights of this neural network are bounded in
absolute value by
\[
\alpha=c_{19} \cdot R^{2},
\]
and consequently this network is contained in $\F(M+1, 2^{M-1} \cdot 16 + \sum_{k=2}^{M} 2^{M-k+1}, c_{19} \cdot R^{2})$.
\end{lemma}

%-----------------------------Proof Lemma 5-----------------------
\begin{proof}
Let $f_{ReLU}$ be the network from Lemma \ref{le3} satisfying
\begin{equation}\label{le5eq1}
\left| f_{ReLU}(x) - \max\{x, 0\}\right| \leq 448 \cdot \frac{\max\{\Vert \sigma^{''} \Vert_{\infty}, \Vert \sigma^{'''} \Vert_{\infty},1\}}{\min\{2 \cdot |\sigma^{'} (t_{\sigma .id})|, |\sigma^{''}(t_{\sigma})|,1\}} \cdot \left(an\right)^3 \cdot \frac{1}{R}
\end{equation}
for $x \in \left[-2an, 2an \right]$.
Since $t_{j+1}-t_j \geq \frac{1}{n}$ all inputs of $f_{B_{j,1,t}}$ are contained in the interval, where \eqref{le5eq1} holds.
Together with
\begin{align*}
B_{j,1,t}(x) &= \frac{x-t_j}{t_{j+1}-t_j} \cdot \mathds{1}_{[t_j, t_{j+1})}(x) + \frac{t_{j+2}-x}{t_{j+2}-t_{j+1}} \cdot \mathds{1}_{[t_{j+1}, t_{j+2})}(x)\\
&= \left(\frac{x-t_j}{t_{j+1}-t_j}\right)_+ - \left(\frac{x-t_{j+1}}{t_{j+1}-t_j}\right)_+ - \left(\frac{x-t_{j+1}}{t_{j+2}-t_{j+1}}\right)_+ + \left(\frac{x-t_{j+2}}{t_{j+2}-t_{j+1}}\right)_+
\end{align*}
we have
\begin{align}\label{le4eq8}
|f_{B_{j,1,t}}(x) - B_{j,1,t}(x)| \leq &\left| f_{ReLU}\left(\frac{x-t_j}{t_{j+1}-t_j}\right) - \left(\frac{x-t_j}{t_{j+1}-t_j}\right)_+\right| \notag \\
&  +  \left|f_{ReLU}\left(\frac{x-t_{j+1}}{t_{j+1}-t_j}\right)-\left(\frac{x-t_{j+1}}{t_{j+1}-t_j}\right)_+\right| \notag \\
&  +  \left|f_{ReLU}\left(\frac{x-t_{j+1}}{t_{j+2}-t_{j+1}}\right) - \left(\frac{x-t_{j+1}}{t_{j+2}-t_{j+1}}\right)_+\right| \notag\\
&  +  \left|f_{ReLU} \left(\frac{x-t_{j+2}}{t_{j+2}-t_{j+1}}\right) - \left(\frac{x-t_{j+2}}{t_{j+2}-t_{j+1}}\right)_+\right| \notag\\
\leq & 4 \cdot 448  \cdot \frac{\max\{\Vert \sigma^{''} \Vert_{\infty}, \Vert \sigma^{'''} \Vert_{\infty},1\}}{\min\{2 \cdot |\sigma^{'} (t_{\sigma .id})|, |\sigma^{''}(t_{\sigma})|,1\}} \cdot \left(an\right)^3 \cdot \frac{1}{R}
\end{align}
and $f_{B_{j,1,t}}$ is contained in $\mathcal{F}(2, 16, c_{20} \cdot R^2)$ for some constant $c_{20} > 0$.\newline

By Lemma \ref{le1} we have
\begin{equation}\label{le4eq5}
|f_{id}(x) - x| \leq \frac{9 \cdot \Vert \sigma^{''} \Vert_{\infty} \cdot (an)^2}{2 \cdot |\sigma^{'}(t_{\sigma.id})|} \cdot \frac{1}{R}
\end{equation}
for $x \in \left[-3an, 3an \right]$ and $f_{mult}$ from Lemma \ref{le2} satisfies
\begin{equation}\label{le4eq7}
|f_{mult}(x,y)-xy| \leq \frac{180 \cdot \Vert \sigma^{'''}\Vert_{\infty} \cdot (an)^3}{|\sigma^{''}(t_{\sigma})|} \cdot \frac{1}{R}
\end{equation}
for $x,y \in \left[-3an, 3an\right]$.
%For any $t \in \N$ and $z \in \R$ we set
%\begin{align*}
%f_{id}^0(z) = z \quad \text{and} \quad f_{id}^{t+1}(z)=f_{id}(f_{id}^t(z))
%\end{align*}
For $x \in \left[-2an, 2an\right]$ and $R > (t-1) \cdot \frac{9  \cdot \Vert \sigma^{''}\Vert_{\infty} \cdot (an)^2}{2 \cdot |\sigma^{'}(t_{\sigma.id})|}$ we get from \eqref{le4eq5} that
\begin{align*}
|f_{id}^t(x)-x| &\leq \sum_{k=1}^t |f_{id}^k(x)-f_{id}^{k-1}(x)|\\
&= \sum_{k=1}^t| f_{id}(f_{id}^{k-1}(x)) - f_{id}^{k-1}(x)|\\
&\leq t \cdot  \frac{9  \cdot \Vert \sigma^{''}\Vert_{\infty} \cdot (an)^2}{2 \cdot |\sigma^{'}(t_{\sigma.id})|} \cdot \frac{1}{R}.
\end{align*}
For $R >  M \cdot \frac{9  \cdot \Vert \sigma^{''}\Vert_{\infty} \cdot (an)^2}{2 \cdot |\sigma^{'}(t_{\sigma.id})|}$ we can conclude, that 
\begin{equation}\label{le4eq6}
\left|f_{id}^{l}\left(\frac{x-t_j}{t_{j+l}-t_j}\right)\right| \leq 3an \quad \text{and} \quad \left|f_{id}^{l}\left(\frac{t_{j+l+1} -x}{t_{j+l+2}-t_{j+1}}\right)\right| \leq 3an.
\end{equation}
for all $l \in \{1, \dots, M\}$.
When $l>1$ we have
\begin{align*}
f_{B_{j,l+1,t}}(x) = &f_{mult}\left(f_{id}^{l+1} \left(\frac{x-t_j}{t_{j+l+1}-t_j}\right), f_{B_{j,l,t}}(x)\right)\\
& + f_{mult}\left(f_{id}^{l+1}\left(\frac{t_{j+l+2}-x}{t_{j+l+2}-t_{j+1}}\right), f_{B_{j+1, l,t}}(x)\right).
\end{align*}
In the sequel we show by induction, that
\begin{equation}\label{le4eq9}
|f_{B_{j,l,t}}(x) - B_{j,l,t}(x)| \leq (4 \cdot 3 \cdot l)^{l-1} \cdot \left(an\right)^{l+2} \cdot 4 \cdot 448 \cdot \frac{\max\{\Vert \sigma^{''} \Vert_{\infty}, \Vert \sigma^{'''} \Vert_{\infty},1\}}{\min\{2 \cdot |\sigma^{'}(t_{\sigma.id})|, |\sigma^{''}(t_{\sigma})|,1\}} \cdot \frac{1}{R}
\end{equation}
for $l = 1, \dots, M$.
For $l=1$ inequality \eqref{le4eq9} follows by \eqref{le4eq8}. Assume now that \eqref{le4eq9} holds for some $l \in \{1, \dots, M-1\}$. We have 
\begin{align*}
&|f_{B_{j,l+1,t}}(x) - B_{j,l+1,t}(x)|\\
= &\left| f_{mult}\left(f_{id}^{l+1} \left(\frac{x-t_j}{t_{j+l+1}-t_j}\right), f_{B_{j,l,t}}(x)\right) + f_{mult}\left(f_{id}^{l+1} \left(\frac{t_{j+l+2}-x}{t_{j+l+2}-t_{j+1}}\right), f_{B_{j+1,l,t}}(x)\right)\right.\\
& \left. \quad - \frac{x-t_{j}}{t_{j+l+1}- t_j} \cdot B_{j,l,t}(x) - \frac{t_{j+l+2}-x}{t_{j+l+2}-t_{j+1}} \cdot B_{j+1,l,t}(x)\right|\\
\leq & \left| f_{mult}\left(f_{id}^{l+1} \left(\frac{x-t_j}{t_{j+l+1}-t_j}\right), f_{B_{j,l,t}}(x)\right) - f_{id}^{l+1} \left(\frac{x-t_j}{t_{j+l+1}-t_j}\right) \cdot f_{B_{j,l,t}}(x)\right|\\
& \quad + \left|f_{id}^{l+1} \left(\frac{x-t_j}{t_{j+l+1}-t_j}\right) \cdot f_{B_{j,l,t}}(x) - \frac{x-t_j}{t_{j+l+1}-t_j}\cdot f_{B_{j,l,t}}(x)\right|\\
& \quad + \left|\frac{x-t_j}{t_{j+l+1}-t_j} \cdot f_{B_{j,l,t}}(x) - \frac{x-t_j}{t_{j+l+1}-t_j} \cdot B_{j,l,t}(x)\right|\\
& + \left| f_{mult}\left(f_{id}^{l+1} \left(\frac{t_{j+l+2}-x}{t_{j+l+2}-t_{j+1}}\right), f_{B_{j+1,l,t}}(x)\right) - f_{id}^{l+1} \left(\frac{t_{j+l+2}-x}{t_{j+l+2}-t_{j+1}}\right) \cdot f_{B_{j+1,l,t}}(x)\right|\\
& \quad + \left|f_{id}^{l+1} \left(\frac{t_{j+l+2}-x}{t_{j+l+2}-t_{j+1}}\right) \cdot f_{B_{j+1,l,t}}(x) - \frac{t_{j+l+2}-x}{t_{j+l+2}-t_{j+1}}\cdot f_{B_{j+1,l,t}}(x)\right|\\
& \quad + \left|\frac{t_{j+l+2} -x}{t_{j+l+2}-t_{j+1}} \cdot f_{B_{j+1,l,t}}(x) - \frac{t_{j+l+2}-x}{t_{j+l+2}-t_{j+1}} \cdot B_{j+1,l,t}(x)\right|.\\
\end{align*}
For $R \geq (4 \cdot 3 \cdot l)^{l-1} \cdot \left(an\right)^{l+2} \cdot 4 \cdot 448 \cdot \frac{\max\{\Vert \sigma^{''} \Vert_{\infty}, \Vert \sigma^{'''} \Vert_{\infty},1\}}{\min\{2 \cdot |\sigma^{'}(t_{\sigma.id})|, |\sigma^{''}(t_{\sigma})|,1\}}$ we can conclude with Lemma 14.2 and 14.3 in \cite{GKKW02} and the induction hypothesis, that 
\begin{align*}
|f_{B_{j,l,t}}(x)| &\leq |f_{B_{j,l,t}}(x)-B_{j,l,t}(x)| + |B_{j,l,t}(x)| \leq 2
\end{align*}
and analogously 
\begin{align}
|f_{B_{j+1,l,t}}(x)| \leq 2.
\end{align}
Together with \eqref{le4eq6} we can conclude, that $f_{id}^{l}\left(\frac{x-t_j}{t_{j+l+1}-t_j}\right)$, $f_{id}^{l}\left(\frac{t_{j+l+2}-x}{t_{j+l+2}-t_{j+1}}\right)$,  $f_{B_{j,l,t}}(x)$ and $f_{B_{j+1,l,t}}(x)$ are contained in the interval, where \eqref{le4eq7} holds. This together with \eqref{le4eq8} and the induction hypothesis leads to
\begin{align*}
&|f_{B_{j,l+1,t}}(x) - B_{j,l+1,t}(x)|\\
\leq & \frac{180 \cdot \Vert \sigma^{'''}\Vert_{\infty} \cdot (an)^3}{|\sigma^{''}(t_{\sigma})|} \cdot \frac{1}{R} + |f_{B_{j,l,t}}(x)| \cdot (l+1) \cdot \frac{9 \cdot \Vert \sigma^{''} \Vert_{\infty} \cdot (an)^2}{2 \cdot |\sigma^{'}(t_{\sigma.id})|} \cdot \frac{1}{R}\\
& \quad + \left|\frac{x-t_j}{t_{j+l+1}-t_j}\right| \cdot \left|f_{B_{j,l,t}}(x) - B_{j,l,t}(x)\right| + \frac{180 \cdot \Vert \sigma^{'''}\Vert_{\infty} \cdot (an)^3}{|\sigma^{''}(t_{\sigma})|} \cdot \frac{1}{R}\\
& \quad +  |f_{B_{j+1,l,t}}(x)| \cdot (l+1) \cdot \frac{9 \cdot \Vert \sigma^{''} \Vert_{\infty} \cdot (an)^2}{2 \cdot |\sigma^{'}(t_{\sigma.id})|} \cdot \frac{1}{R} + \left|\frac{t_{j+l+2}-x}{t_{j+l+2}-t_{j+1}}\right| \cdot |f_{B_{j+1,l,t}}(x) - B_{j+1,l,t}(x)|\\
\leq & 6 \cdot  (l+1) \cdot 2an \cdot (4 \cdot 3 \cdot l)^{l-1} \cdot \left(an\right)^{l+2}\cdot  4 \cdot 448 \cdot \frac{\max\{\Vert \sigma^{''} \Vert_{\infty}, \Vert \sigma^{'''} \Vert_{\infty},1\}}{\min\{2 \cdot |\sigma^{'}(t_{\sigma.id})|, |\sigma^{''}(t_{\sigma})|,1\}} \cdot \frac{1}{R}\\
\leq & (4 \cdot 3 \cdot (l+1))^{l} \cdot \left(an\right)^{l+3} \cdot  4 \cdot 448 \cdot \frac{\max\{\Vert \sigma^{''} \Vert_{\infty}, \Vert \sigma^{'''} \Vert_{\infty},1\}}{\min\{2 \cdot |\sigma^{'}(t_{\sigma.id})|, |\sigma^{''}(t_{\sigma})|,1\}} \cdot \frac{1}{R},
\end{align*}
which shows the assertion. 
\end{proof}

%\begin{remark}
%\cite{Liu19} presented a similar approximation result with neural networks using the ReLU activation function. Since in this result $x$ is contained in the interval $[0,1]$ and the activation function is non-smooth, we could not use it in our approximation. 
%\end{remark}

%------------------------------Lemma 6--------------------------------------------------
\begin{lemma}
  \label{le6}
Let $K \in \N$, $f_1, \dots, f_K:\R^d \rightarrow \R$
and $a \geq 1$.
Let $r \in \N$, $\alpha \geq 1$ and let
 $f_{net,1}$, \dots, $f_{net,K}$ be neural networks satisfying
\[
f_{net,k} \in \F(L_k,r,\alpha)
\quad (k=1, \dots, K).
\]
Let $0\leq \epsilon_k \leq 1 \leq \beta_k$ be such that
\begin{equation}
\label{le6eq1}
| f_{net,k}(x)-f_k(x)| \leq \epsilon_k
\quad \mbox{for all } x \in [-2a,2a]^d
\end{equation}
and
\begin{equation}
\label{le5eq2}
|f_k(x)| \leq \beta_k
\quad \mbox{for all } x \in [-2a,2a]^d
\end{equation}
$(k=1, \dots, K)$. Let $C \geq 1$ be such that
\begin{equation}
\label{le5eq3}
|f_k(x)-f_k(z)| \leq C \cdot \|x-z\|
\quad \mbox{for all } x,z \in [-2a,2a]^d, k \in \{1, \dots,K\}.
\end{equation}
Then there exists a neural network
\[
f_{prod}
\in
\F \left(
K-1+\sum_{k=1}^K L_k,
r+d+5,
\alpha_{prod}
\right)
\]
with
\[
\alpha_{prod}=c_{21} \cdot \alpha
\cdot a^6 \cdot K^{8} \cdot 2^{8K} \cdot (\prod_{j=1}^K \beta_j)^{8}
\cdot d^2\cdot C^2 \cdot
(L_1+\dots+L_K+K-1)^2\cdot  n^6,
\]
which satisfies
\[
|
f_{prod}(x)
-
\prod_{k=1}^K f_k(x)
|
\leq
\max\left\{
K \cdot 2^{K}
\cdot
(\prod_{j=1}^K \beta_j)
\cdot
\max\left\{\epsilon_1,\dots,\epsilon_K\right\}, \frac{1}{n^3}\right\}
\]
for all $x \in [-a,a]^d$
  \end{lemma}

%-------------------------------------Proof of Lemma 6-----------------------------
\noindent
\begin{proof}
In the construction of $f_{prod}$ we will use the networks
$f_{id}$ from Lemma \ref{le1} and $f_{mult}$ from
Lemma \ref{le2} chosen such that
\[
| f_{id}(z)-z|
\leq
\frac{ 2 \cdot
\| \sigma^{\prime \prime}\|_{\infty} \cdot (\max\{a, K \cdot 2^K \cdot \prod_{j=1}^K \beta_j\})^2
}{
 |\sigma^\prime(t_{\sigma,id})|
}
\cdot
\frac{1}{R_{id}}
\]
for
$z \in [-2 \max\{a, K \cdot 2^K \cdot \prod_{j=1}^K \beta_j\},2 \max\{a, K \cdot 2^K \cdot \prod_{j=1}^K \beta_j\}] $
and
\[
|f_{mult}(z_1,z_2) - z_1 \cdot z_2| \leq
\frac{160 \cdot \|
\sigma^{\prime \prime \prime}
\|_{\infty} \cdot (\max\{a, K \cdot 2^K \cdot \prod_{j=1}^K \beta_j\})^3}{
3 \cdot |\sigma^{ \prime \prime} (t_\sigma)|}
 \cdot \frac{1}{R_{mult}}
\]
for
$z_1,z_2 \in [-2 \max\{a, K \cdot 2^K \cdot \prod_{j=1}^K \beta_j\},2 \max\{a, K \cdot 2^K \cdot \prod_{j=1}^K \beta_j\}]$.
Here we will use
\begin{eqnarray*}
R_{id}&=& \frac{
2 \cdot
\| \sigma^{\prime \prime}\|_{\infty} \cdot (\max\{a, K \cdot 2^K \cdot \prod_{j=1}^K \beta_j\})^2
}{
|\sigma^\prime(t_{\sigma,id})|
}
\\
&&
\hspace*{3cm}
\cdot
4d \cdot C \cdot (L_1+\dots+L_K+K-1) \cdot n^3 \cdot K \cdot 2^K \cdot \prod_{j=1}^K \beta_j
\end{eqnarray*}
and 
\begin{eqnarray*}
R_{mult}&=&\frac{
160 \cdot \|
\sigma^{\prime \prime \prime}
\|_{\infty} \cdot (\max\{a, K \cdot 2^K \cdot \prod_{j=1}^K \beta_j\})^3
}{
3 \cdot |\sigma^{ \prime \prime} (t_\sigma)|
}
\\
&&
\hspace*{3cm}
\cdot
4d \cdot C \cdot (L_1+\dots+L_K+K-1) \cdot n^3 \cdot K \cdot 2^K \cdot \prod_{j=1}^K \beta_j
\end{eqnarray*}
which implies that the network $f_{id}$ satisfies
\begin{equation}
\label{ple5eq2}
| f_{id}(z)-z|
\leq
\frac{1}{4d \cdot C \cdot (L_1+\dots+L_K+K-1) \cdot n^3 \cdot K \cdot 2^K \cdot \prod_{j=1}^K \beta_j}
\end{equation}
for
$z \in [-2 \max\{a, K \cdot 2^K \cdot \prod_{j=1}^K \beta_j\},2 \max\{a, K \cdot 2^K \cdot \prod_{j=1}^K \beta_j\}]$
and that all its weights are bounded in absolute value by $
c_{22} \cdot a^2 \cdot K^3 \cdot 2^{3K} \cdot (\prod_{j=1}^K \beta_j)^3 \cdot
d \cdot C \cdot (L_1+\dots+L_K+K) \cdot n^3
$, and
that
the network $f_{mult}$ satisfies
\begin{equation}
\label{ple5eq3}
|f_{mult}(z_1,z_2) - z_1 \cdot z_2| \leq
\frac{1}{4d \cdot C \cdot (L_1+\dots+L_K+K-1) \cdot n^3 \cdot K \cdot 2^K \cdot \prod_{j=1}^K \beta_j}
\end{equation}
for
$z_1,z_2\in [-2 \max\{a, K \cdot 2^K \cdot \prod_{j=1}^K \beta_j\},2 \max\{a, K \cdot 2^K \cdot \prod_{j=1}^K \beta_j\}]$
and that all its weights are bounded in absolute value by $
c_{23} \cdot a^6 \cdot K^8 \cdot 2^{8K} \cdot (\prod_{j=1}^K \beta_j)^8
\cdot d^2 \cdot C^2 \cdot (L_1+\dots+L_K+K)^2 \cdot n^6
$.
%For $t \in \N_0$ and $z \in \R$ set
%\[
%f_{id}^0(z)=z \quad \mbox{and} \quad f_{id}^{t+1}(z)=f_{id}(f_{id}^t(z)),
%\]
%and for $x=(x^{(1)},\dots,x^{(d)}) \in \Rd$ set
%\[
%f_{id}^t(x)=(f_{id}^t(x^{(1)}), \dots, f_{id}^t(x^{(d)})).
%\]
We start the construction of $f_{prod}$ by defining a neural network with
$L=K-1+\sum_{k=1}^K L_k$
hidden layers and
$k_1=\dots=k_L=r+d+5$
neurons in each layer by (6)--(8).

Neurons $1, \dots, d$ of this network will be used to provide the
value
of the input in layers $2, 3, \dots, L$. To achieve this, we
copy the network $f_{id}$ in neuron $k$ in layer $1$,
\dots, $L-1$, where $f_{id}$ gets in layer one the $k$--th component
of $x$ as input, and where $f_{id}$ gets in all other layers as input
the output of $f_{id}$ from the previous layer $(k=1, \dots, d)$.

Neurons $d+1$, \dots, $d+r$ will be used to compute $f_{net,1}$,
\dots, $f_{net,K}$. To achieve this, we copy these networks
successively
in the neurons $d+1$, \dots, $d+r$ in layers $1$, \dots, $L$
such that $f_{net,1}$ is contained in layers $1$ till $L_1$, and for
$k \in \{2,\dots,K\}$ the network
$f_{net,k}$ is contained in layers $L_1+L_2+\dots+L_{k-1}+k-1$,\dots,
$L_1+\dots+L_k+k-2$. Here $f_{net,1}$ gets as
input the input of our network, while all other networks
$f_{net,k}$
get as
input the value of the input provided by the neurons $1, \dots, d$
in the layer before the network $f_{net,k}$  starts.

Neurons $d+r+1$, \dots, $d+r+4$ are used to compute the product
of the networks $f_{net,1}$, \dots, $f_{net,K}$, and
neuron
$d+r+5$ is  used to provide the value of the part
of the product, which is already computed, for the next level.
To achieve this, we copy the network $f_{mult}$ in the
neurons
$d+r+1$, \dots, $d+r+4$
in each of the layers
$L_1+L_2+1$, $L_1+L_2+L_3+2$,\dots,$L_1+L_2+\dots+L_K+K-1$
and we copy the network $f_{id}$ in the neuron
$d+r+5$
in layers $L_1+1$, $L_1+2$, \dots, $L-1$.
Here the network $f_{mult}$ gets as input
the output of the neurons $d+1$, $d+2$ \dots, $d+r$  and
$d+r+5$
in the layer before the network $f_{mult}$ starts.
And the network
$f_{id}$ gets in layer
$L_1+L_2+2$, $L_1+L_2+L_3+3$,\dots,$L_1+L_2+\dots+L_{K-1}+K-1$
the output of the neurons
 $d+r+1$, \dots, $d+r+4$ of the previous layer as input,
in layer $L_1+1$ it gets the output of neurons $d+1$, \dots,
$d+r$
of layer $L_1$ as input,
and in all other
 layers it gets
the output of the neuron $d+r+5$ of the previous layer
as input.

The output of the neural network is the sum of the outputs
of the neurons $d+r+1$, \dots, $d+r+4$.

By construction, this network successively computes
\[
g_1(x)=f_{net,1}(x),
\]
\[
g_2(x)=f_{mult} (f_{id}^{L_2}(g_1(x)) , f_{net,2}(f_{id}^{L_1}(x))),
\]
\[
g_3(x)=f_{mult}(f_{id}^{L_3}(g_2(x)), f_{net,3}(f_{id}^{L_1+L_2+1}(x))),
\]
\hfill \vdots \hfill \quad
\[
g_K(x)=f_{mult}(f_{id}^{L_{K}}(g_{K-1}(x)),
f_{net,K}(f_{id}^{L_1+L_2+\dots+L_{K-1} + K-2}(x))).
\]
The output of the network is
\[
f_{prod}(x)=g_{K}(x).
\]
In the sequel we show by induction
\begin{equation}
\label{ple5eq1}
\left|
g_k(x)-\prod_{j=1}^k f_j(x)
\right|
\leq
k \cdot 2^{k}
\cdot
\left(
\prod_{j=1}^k \beta_j
\right)
\cdot
\max\left\{\epsilon_1,\dots,\epsilon_k, \frac{1}{n^3 \cdot K \cdot 2^K \cdot \prod_{j=1}^K \beta_j}\right\}
\end{equation}
for $ k \in \{1, \dots, K\}$.

For $k=1$ inequality (\ref{ple5eq1}) follows from (\ref{le6eq1}).
Assume now that
(\ref{ple5eq1}) holds for some $k \in \{1, \dots, K-1\}$.
We have
\begin{eqnarray*}
&&
\left|
g_{k+1}(x)-\prod_{j=1}^{k+1} f_j(x)
\right|
\\
&&
=
\left|
f_{mult}( f_{id}^{L_{k+1}}(g_{k}(x)) , f_{net,k+1}(f_{id}^{L_1+L_2+\dots+L_{k}+k-1}(x)))
- f_{k+1}(x) \cdot \prod_{j=1}^{k} f_j(x)
\right|
\\
&&
\leq
\bigg|
f_{mult}( f_{id}^{L_{k+1}}(g_{k}(x)) ,
   f_{net,k+1}(f_{id}^{L_1+L_2+\dots+L_{k}+k-1}(x)))
\\
&&
\hspace*{3cm}
-
f_{id}^{L_{k+1}}(g_{k}(x))  \cdot f_{net,k+1}(f_{id}^{L_1+L_2+\dots+L_{k}+k-1}(x)))
\bigg|
\\
&&
\quad
+
\bigg|
f_{id}^{L_{k+1}}(g_{k}(x))\cdot
   f_{net,k+1}(f_{id}^{L_1+L_2+\dots+L_{k}+k-1}(x))
\\
&&
\hspace*{3cm}
-
g_{k}(x) \cdot f_{net,k+1}(f_{id}^{L_1+L_2+\dots+L_{k}+k-1}(x))
\bigg|
\\
\end{eqnarray*}
\begin{eqnarray*}
&&
\quad
+
\bigg|
g_{k}(x)
\cdot
  f_{net,k+1}(f_{id}^{L_1+L_2+\dots+L_{k}+k-1}(x))
\\
&&
\hspace*{3cm}
-
(\prod_{j=1}^{k} f_j(x))
\cdot
  f_{net,k+1}(f_{id}^{L_1+L_2+\dots+L_{k}+k-1}(x))
\bigg|
\\
&&
\quad
+
\bigg|
(\prod_{j=1}^{k} f_j(x))
\cdot
  f_{net,k+1}(f_{id}^{L_1+L_2+\dots+L_{k}+k-1}(x))
\\
&&
\hspace*{3cm}
-(\prod_{j=1}^{k} f_j(x))
\cdot
  f_{k+1}(f_{id}^{L_1+L_2+\dots+L_{k}+k-1}(x))
\bigg|
\\
&&
\quad
+
\left|
(\prod_{j=1}^{k} f_j(x))
\cdot
  f_{k+1}(f_{id}^{L_1+L_2+\dots+L_{k}+k-1}(x))
-(\prod_{j=1}^{k} f_j(x))
\cdot
  f_{k+1}(x)
\right|.
\end{eqnarray*}
For any $t \in \{1, \dots, L_1+ \dots + L_K + K -1\}$ and 
\begin{align*}
z \in
[- \max\{a, K \cdot 2^K \cdot \prod_{j=1}^K \beta_j\},  \max\{a, K
\cdot 2^K \cdot \prod_{j=1}^K \beta_j\}]
\end{align*}
we get from (\ref{ple5eq2}) that
\begin{eqnarray}
\left|
f_{id}^t(z)-z
\right|
&
\leq
&
\sum_{k=1}^t |f_{id}^k(z)-f_{id}^{k-1}(z)|
=
\sum_{k=1}^t |f_{id}(f_{id}^{k-1}(z))-f_{id}^{k-1}(z)|
\nonumber \\
&\leq & \frac{t}{ 4d \cdot C \cdot (L_1+\dots+L_K+K-1) \cdot n^3 \cdot K \cdot 2^K \cdot \prod_{j=1}^K \beta_j},
\label{ple5eq4}
\end{eqnarray}
which implies that for $x \in [-a,a]^d$ we have
\[
f_{id}^{L_1+L_2+\dots+L_{k}+k-1}(x) \in [-2a,2a]^d.
\]
Similarly we can conclude from  the induction hypothesis
and (\ref{le5eq2}), which
imply
\begin{equation}
\label{ple5eq5}
|g_k(x)| \leq |g_k(x)-\prod_{j=1}^{k}f_j(x)| + \prod_{j=1}^k \beta_j
\leq
 K \cdot 2^K \cdot \prod_{j=1}^K \beta_j
\end{equation}
 that we also have
\[
f_{id}^{L_{k+1}}(g_{k}(x))
\in
[-2 \max\{a,  K \cdot 2^K \cdot \prod_{j=1}^K \beta_j\}, 2 \max\{a,  K \cdot 2^K \cdot \prod_{j=1}^K \beta_j\}].
\]

By (\ref{le5eq2}), (\ref{le5eq3}) and (\ref{ple5eq4}) we get
\begin{eqnarray*}
&&
\left|
(\prod_{j=1}^{k} f_j(x))
\cdot
  f_{k+1}(f_{id}^{L_1+L_2+\dots+L_{k}+k-1}(x))
-(\prod_{j=1}^{k} f_j(x))
\cdot
  f_{k+1}(x)
\right|
\\
&&
\leq
(\prod_{j=1}^{k} \beta_j) \cdot
C \cdot \|f_{id}^{L_1+L_2+\dots+L_{k}+k-1}(x)-x\|
\\
&&
\leq \frac{d \cdot
C \cdot (\prod_{j=1}^{k} \beta_j) \cdot
(L_1+L_2+\dots+L_{k}+k-1)
}{
4d \cdot C \cdot (L_1+\dots+L_K+K-1) \cdot n^3 \cdot K \cdot 2^K \cdot \prod_{j=1}^K \beta_j
}
\leq \frac{\prod_{j=1}^k \beta_j}{4n^3 \cdot K \cdot 2^K \cdot \prod_{j=1}^K \beta_j}.
\end{eqnarray*}
In addition we have
\begin{eqnarray*}
&&
\left|
(\prod_{j=1}^{k} f_j(x))
\cdot
  f_{net,k+1}(f_{id}^{L_1+L_2+\dots+L_{k}+k-1}(x))
-(\prod_{j=1}^{k} f_j(x))
\cdot
  f_{k+1}(f_{id}^{L_1+L_2+\dots+L_{k}+k-1}(x))
\right|
\\
&&
\leq
(\prod_{j=1}^k \beta_j) \cdot \epsilon_{k+1}.
\end{eqnarray*}
By a similar argument we get
\begin{eqnarray*}
&&
|
f_{net,k+1}(f_{id}^{L_1+L_2+\dots+L_{k}+k-1}(x))|
\\
&&
\leq
|
f_{net,k+1}(f_{id}^{L_1+L_2+\dots+L_{k}+k-1}(x))
-
f_{k+1}(f_{id}^{L_1+L_2+\dots+L_{k}+k-1}(x))
|
\\
&&
\quad
+
| f_{k+1}(f_{id}^{L_1+L_2+\dots+L_{k}+k-1}(x))
|
\\
&&
\leq
\epsilon_{k+1}+\beta_{k+1}.
\end{eqnarray*}
This together with the induction hypothesis implies
\begin{eqnarray*}
&&
\left|
g_{k}(x)
\cdot
  f_{net,k+1}(f_{id}^{L_1+L_2+\dots+L_{k}+k-1}(x))
-
(\prod_{j=1}^{k} f_j(x))
\cdot
  f_{net,k+1}(f_{id}^{L_1+L_2+\dots+L_{k}+k-1}(x))
\right|
\\
&&
\leq k \cdot 2^{k}
\cdot
(\prod_{j=1}^k \beta_j)
\cdot
\max\left\{\epsilon_1,\dots,\epsilon_k, \frac{1}{n^3 \cdot K \cdot 2^K \cdot \prod_{j=1}^K \beta_j}\right\} \cdot (\epsilon_{k+1}+\beta_{k+1}).
\end{eqnarray*}
Using (\ref{ple5eq5})
we can apply (\ref{ple5eq4}) once more and conclude also that
\begin{eqnarray*}
&&
\left|
f_{id}^{L_{k+1}}(g_{k}(x)) \cdot f_{net,k+1}(f_{id}^{L_1+L_2+\dots+L_{k}+k-1}(x))
-
g_{k}(x) \cdot f_{net,k+1}(f_{id}^{L_1+L_2+\dots+L_{k}+k-1}(x))
\right|
\\
&&
\leq
\frac{L_{k+1}}{4d \cdot C \cdot (L_1+\dots+L_K+K-1) \cdot n^3 \cdot K \cdot 2^K \cdot \prod_{j=1}^K \beta_j} \cdot
(\epsilon_{k+1} + \beta_{k+1})
\\&&
\leq
\frac{\epsilon_{k+1}+\beta_{k+1}}{4n^3 \cdot K \cdot 2^K \cdot \prod_{j=1}^K \beta_j}.
\end{eqnarray*}
And finally we see that
$g_{k}(x)$ and $f_{net,k+1}(f_{id}^{L_1+L_2+\dots+L_{k}}(x)$ are both
contained
in the interval where (\ref{ple5eq3}) holds, which implies
\begin{eqnarray*}
&&
\bigg|
f_{mult}( f_{id}^{L_{k+1}}(g_{k}(x)) , f_{net,k+1}(f_{id}^{L_1+L_2+\dots+L_{k}}(x)))
\\
&&
\hspace*{3cm}-
f_{id}^{L_{k+1}}(g_{k}(x))  \cdot f_{net,k+1}(f_{id}^{L_1+L_2+\dots+L_{k}}(x)))
\bigg|
\\
&&
\leq
\frac{1}{4n^3 \cdot K \cdot 2^K \cdot \prod_{j=1}^K \beta_j}.
\end{eqnarray*}
Summarizing the above result we get
\begin{align*}
&
\left|
g_{k+1}(x)-\prod_{j=1}^{k+1} f_j(x)
\right|
\\
&
\leq
\frac{1}{4n^3 \cdot K \cdot 2^K \cdot \prod_{j=1}^K \beta_j}
+
\frac{1}{4n^3 \cdot K \cdot 2^K \cdot \prod_{j=1}^K \beta_j}
\cdot  (\epsilon_{k+1} + \beta_{k+1})
\\
&
\quad
+
k \cdot 2^{k}
\cdot
(\prod_{j=1}^k \beta_j)
\cdot
\max\left\{\epsilon_1,\dots,\epsilon_k, \frac{1}{n^3 \cdot K \cdot 2^K \cdot \prod_{j=1}^K \beta_j}\right\}
 \cdot (\epsilon_{k+1}+\beta_{k+1})
\\
&
\quad
+
(\prod_{j=1}^k \beta_j) \cdot \epsilon_{k+1}
+
\frac{
\prod_{j=1}^k \beta_j
}{4n^3 \cdot K \cdot 2^K \cdot \prod_{j=1}^K \beta_j}
\end{align*}
\begin{align*}
&\leq
2 \cdot (\prod_{j=1}^{k+1} \beta_j) \cdot \max\left\{\epsilon_{k+1}, \frac{1}{n^3 \cdot K \cdot 2^K \cdot \prod_{j=1}^K \beta_j}\right\}
\\
&
\quad
+ 2 \cdot  k \cdot 2^{k}
\cdot
(\prod_{j=1}^{k+1} \beta_j)
\cdot
\max\left\{\epsilon_1,\dots,\epsilon_k, \frac{1}{n^3 \cdot K \cdot 2^K \cdot \prod_{j=1}^K \beta_j}\right\}
\\
&
\leq
(k+1) \cdot 2^{k+1}
\cdot
(\prod_{j=1}^{k+1} \beta_j)
\cdot
\max\left\{\epsilon_1,\dots,\epsilon_{k+1}, \frac{1}{n^3 \cdot K \cdot 2^K \cdot \prod_{j=1}^K \beta_j}\right\}.
\end{align*}
The proof of (\ref{ple5eq1}) is
complete, which implies the assertion of Lemma \ref{le6}.
\end{proof}

\begin{lemma}\label{le8}
Let $a \geq 1$, $M, K, K_1\in \N$, $J \leq d$, $\alpha_{k,j} \in \R$, $\gamma_{k,j} \in [-a,a]$ $(j=1, \dots, d, k=1, \dots, K_1)$ and $j_{v} \in \{-M, \dots, K-1\}$ $(v=1, \dots, J)$. Let $B_{j_v,M,t_v}$ be a B-Spline of degree $M$  according to Definition 4 with knot sequence $t_v = \{t_{v,k}\}_{k=-M, \dots, M+K}$ such that $t_{v,k} \in [-a,a]$ and $t_{v,k+1}-t_{v,k} \geq \frac{1}{n}$. Set
\begin{equation*}
f(x) = \prod_{v=1}^{J} B_{j_v,M,t_v}(x^{(v)}) \cdot \prod_{k=1}^{K_1} \max\left\{\sum_{j=1}^d \alpha_{k,j} \cdot (x^{(j)} - \gamma_{k,j}),0\right\}.
\end{equation*}
Then there exists a neural network
\begin{equation*}
f_{basis} \in \mathcal{F}(3K_1 +J \cdot (M+2)-1, 2^{M-1} \cdot 16 + \sum_{k=2}^M 2^{M-k+1} + d+5, \alpha_{basis}),
\end{equation*}
where
\begin{align*}
\alpha_{basis} = &c_{24} \cdot 4^{3M+4 \cdot (J+K_1)} \cdot 9^{K_1+M} \cdot M^{2 \cdot (M-1)} \cdot a^{2 \cdot (M+K_1+6)}\\
&\cdot d^{2K_1+17} \cdot (J+K_1)^{10} \cdot n^{2M+16} \cdot \max\{\max_{k,j}|\alpha_{k,j}|,1\}^{2K_1+15}\\
 &\cdot \max\{d \cdot \max_{k,j} |\alpha_{k,j}|,n\}^2 \cdot ((M+2) \cdot J + 3K_1)^2
\end{align*}
such that 
\begin{equation*}
|f(x) - f_{basis}(x)| \leq \frac{1}{n^3}
\end{equation*}
for all $x \in [-a,a]^d$.
\end{lemma}
%-----------------------------------------Proof of Lemma 8---------------
\begin{proof}
Set
\begin{align*}
f_k(x) = B_{j_k,M,t_k}(x^{(k)}) \ \mbox{for} \ k=1, \dots, J
\end{align*}
and
\begin{align*}
f_k(x) = \max\left\{\sum_{j=1}^d \alpha_{k,j} \cdot (x^{(j)} - \gamma_{k,j}),0\right\} \ \mbox{for} \ k=J+1, \dots, J+K_1.
\end{align*}
By Lemma \ref{le5} there exists a neural network
\begin{align*}
f_{B_{j_k,M,t_k}} \in \mathcal{F}\left(M+1, 2^{M-1} \cdot 16+\sum_{k=2}^M 2^{M-k+1}, c_{25} \cdot R_B^2\right),
\end{align*}
satisfying
\begin{align*}
&|f_{B_{j_k,M,t_k}} - f_k(x)|\\
\leq & (4 \cdot 3 \cdot M)^{M-1} \cdot (2an)^{M+2} \cdot 4 \cdot 448 \cdot \frac{\max\{\Vert \sigma^{''} \Vert_{\infty}, \Vert \sigma^{'''}\Vert_{\infty},1\}}{\min\{2 \cdot |\sigma^{'}(t_{\sigma.id})|, |\sigma^{''}(t_{\sigma})|,1\}} \cdot \frac{1}{R_B}
\end{align*}
for $k=1, \dots, J$ and  $x \in [-2a,2a]$.
Here we will use 
\begin{align*}
R_B &= (4 \cdot 3 \cdot M)^{M-1} \cdot (2an)^{M+2} \cdot 4 \cdot 448 \cdot \frac{\max\{\Vert \sigma^{''}\Vert_{\infty}, \Vert \sigma^{'''} \Vert_{infty},1\}}{\min\{2 \cdot |\sigma^{'}(t_{\sigma.id}|, |\sigma^{''}|,1\}}\\
& \cdot (J+K_1) \cdot 2^{J+K_1} \cdot (3d \cdot \max\{\max_{k,j} |\alpha_{k,j}|,1\} \cdot a)^{K_1} \cdot n^3,
\end{align*}
which implies that the network $f_{B_{j_k,M,t_k}}$ satisfies
\begin{align*}
&|f_{B_{j_k,M,t_k}}(x) - B_{j_k,M,t_k}(x^{(k)})|\\
\leq & \frac{1}{(J+K_1) \cdot 2^{J+K_1}  \cdot (3d \cdot \max\{\max_{k,j} |\alpha_{k,j}|,1\} \cdot a)^{K_1} \cdot n^3} := \epsilon_k
\end{align*}
for $k=1, \dots, J$ and all $x \in [-2a,2a]^d$ and that the weights are bounded in absolute value by 
\begin{align*}
\alpha_{B_{j_k,M,t_k}}= &c_{26} \cdot (4 \cdot 3 \cdot M)^{2M-2} \cdot a^{2 \cdot (M+K_1+2)} \cdot n^{2M+10}\\
\quad & \cdot (J+K_1)^2 \cdot 4^{J+K_1+M} \cdot (3d \cdot \max\{\max_{k,j} |\alpha_{k,j}|,1\})^{2K_1}.
\end{align*}
By Lemma \ref{le4} there exists a neural network
\begin{align*}
f_{trunc,k} \in \F(2,4,c_{27} \cdot R_{trunc}^2 \cdot \max\{1, d \cdot \max_{k,j} |\alpha_{k,j}| \cdot a\}),
\end{align*}
satisfying
\begin{align*}
&|f_{trunc,k}(x) - f_k(x)|\\
\leq & 448 \cdot \frac{\max\{\Vert \sigma^{''} \Vert_{\infty}, \Vert \sigma^{'''} \Vert_{\infty},1\}}{\min\{2 \cdot |\sigma^{'}(t_{\sigma.id})|, |\sigma^{''}(t_{\sigma})|,1\}} \cdot d^3 \cdot 8 \cdot a^3 \cdot \max\{\max_{k,j} |\alpha_{k,j}|,1\}^3 \cdot \frac{1}{R_{trunc}}.
\end{align*}
for $k=J+1, \dots, J+K_1$ and $x \in [-2a,2a]^d$.
Here we will use 
\begin{align*}
R_{trunc}=& 448 \cdot \frac{\max\{\Vert \sigma^{''} \Vert_{\infty}, \Vert \sigma^{'''} \Vert_{\infty},1\}}{\min\{2 \cdot |\sigma^{'}(t_{\sigma.id})|, |\sigma^{''}(t_{\sigma})|,1\}} \cdot d^3 \cdot 8 \cdot a^3 \cdot \max\{\max_{k,j} |\alpha_{k,j}|,1\}^3\\
& \cdot (J+K_1) \cdot 2^{J+K_1} \cdot (3d \cdot \max\{\max_{k,j}|\alpha_{k,j}|,1\} \cdot a)^{K_1} \cdot n^3,
\end{align*}
which implies that the network $f_{trunc,k}$ satisfies 
\begin{align*}
&|f_{trunc,k}(x) - f_k(x)|\\
\leq & \frac{1}{(J+K_1) \cdot 2^{J+K_1} \cdot (3d \cdot \max\{\max_{k,j}|\alpha_{k,j}|,1\} \cdot a)^{K_1}} \cdot \frac{1}{n^3} := \epsilon_k
\end{align*}
for $k =J+1, \dots, J+K_1$ and all $x \in [-2a,2a]^d$ and that all weights are bounded in absolute value by
\begin{align*}
\alpha_{trunc} = c_{28} \cdot 4^{J+K_1} \cdot 9^{K_1} \cdot a^{2K_1+7} \cdot (J+K_1)^2 \cdot d^{2K_1+7} \cdot  (\max\{ \max_{k,j} |\alpha_{k,j}|,1\})^{2K_1+7} \cdot n^6.
\end{align*}

For $x,z \in [-2a,2a]^d$ we have
\begin{alignat*}{3}
&|f_k(x)| \leq 1 \ &&\mbox{for} \ k=1, \dots, J,\\
&|f_k(x)| \leq 3d \cdot \max\{\max_{k,j} |\alpha_{k,j}|,1\} \cdot a \ &&\mbox{for} \ k=J+1, \dots, J+K_1,
\end{alignat*}
where we have used Lemma 14.2 and 14.3 in \cite{GKKW02} for the first inequality, and for $k=J+1, \dots, J+K_1$
\begin{eqnarray*}
|f_k(x)-f_k(z)| &\leq&
\left|
\sum_{j=1}^d |\alpha_{k,j}|\cdot |x^{(j)}-z^{(j)}|
\right|
\\
& \leq &
\sqrt{\sum_{j=1}^d |\alpha_{k,j}|^2}
\cdot \|x-z\|.
\end{eqnarray*}
By Lemma 14.6 in \cite{GKKW02} and $0 \leq B_{j_k,M-1,t_k} \leq 1$ $(k=1, \dots, J)$ we can conclude
\begin{align*}
\left|\frac{\partial}{\partial x^{(k)}} B_{j_k,M,t_k}(x^{(k)})\right| &= \left|\frac{M}{t_{k,j_k+M} - t_{k,j_k}} \cdot B_{j_k, M-1,t_k}(x^{(k)})\right.\\
&\quad -\frac{M}{t_{k,j_k+M+1}- t_{k,j_k+1}} \cdot B_{j_k+1, M-1,t_k}(x^{(k)})\Bigg|\\
& \leq \max\left\{\frac{M}{t_{k,j_k+M} - t_{k,j_k}}, \frac{M}{t_{k,j_k+M+1}-t_{k,j_k+1}}\right\}
\end{align*}
and by the the mean value theorem there exists $x^{(k)} < \eta < z^{(k)}$ such that
\begin{align*}
\left|\frac{B_{j_k,M,t_k}(x^{(k)}) - B_{j_k,M,t_k}(z^{(k)})}{x^{(k)}-z^{(k)}}\right| = & \left|\frac{\partial}{\partial \eta} B_{j_k,M,t_k}(\eta)\right|\\\
\leq & \left|\max\left\{\frac{M}{t_{k,j_k+M}-t_{k,j_k}},\frac{M}{t_{k,j_k+M+1}-t_{k,j_k+1}}\right\}\right|\\
\leq & n.
\end{align*}
Lemma \ref{le6} with $\epsilon_k$ as above, $\beta_k=1$ for $k=1, \dots, J$ and $\beta_k=3d \cdot \max\{\max_{k,j}|\alpha_{k,j}|,1\} \cdot a$ for $k=J+1, \dots, J+K_1$ and $C=\max\{d \cdot \max_{k,j} |\alpha_{k,j}|, n\}$ yields the assertion. 
\end{proof}

\begin{lemma}
  \label{le11}
Let $a \geq 1$,
let $K, M, I,K_1 \in \N$, $J \leq d$ and for $i \in \{1, \dots, I\}$,
$j \in \{1, \dots, d\}$ and $k \in \{1, \dots, K_1\}$ let
$w_i \in \R$, $\alpha_{i, k, j} \in \R$ and $\gamma_{i,k, j} \in
[-a,a]$.
Let $j_{i,v} \in \{-M, \dots, K-1\}$ $(i \in \{1, \dots, I\}, v \in \{1, \dots, J\})$ and let $B_{j_{i,v},M,t_{i,v}}: \R \to \R$ be a B-Spline of degree $M$ according to Definition 4 with knot sequence $t_{i,v} = \{t_{i,v,k}\}_{k=-M, \dots, K+M}$ such that $t_{i,v,k} \in [-a,a]$ and $t_{i,v,k+1}-t_{i,v,k} \geq \frac{1}{n}$.
Set
\[
f(x)=\sum_{i=1}^I w_i \prod_{v=1}^J B_{j_{i,v}, M, t_{i,v}}(x^{(v)}) \cdot \prod_{k=1}^{K_1} \l \max \left\{\sum_{j=1}^d \alpha_{i,k,j} \cdot (x^{(j)} - \gamma_{i,k,j}), 0\right\}.
\] 
Then there exists a neural network
\[
f_{lcb} \in \F^{(sparse)}_{I, L, r, \alpha_{lcb}},
\]
where $L=3K_1+J \cdot (M+2)-1$, $r=2^{M-1} \cdot 16 + \sum_{k=2}^M 2^{M-k+1} + d+ 5$ and 
\begin{equation*}
\alpha_{lcb} = \max \{\alpha_{basis}, \max\{|w_1|, \dots, |w_I|\}\}
\end{equation*}
with $\alpha_{basis}$ as in Lemma \ref{le8}, 
such that
\[
|f_{lcb}(x)-f(x)|
\leq
I \cdot \max\{|w_1|,\dots,|w_I| \} \cdot
\frac{1}{n^3}
\]
for all $x \in [-a,a]^d$.
  \end{lemma}

  \begin{proof}
  According to Lemma \ref{le8} there exists a neural network
  \begin{equation*}
  f_{basis,i} \in \F(3K_1+J \cdot (M+2)-1, 2^{M-1} \cdot 16+ \sum_{k=2}^M 2^{M-k+1}+d+5, \alpha_{basis}),
  \end{equation*}
  where 
\begin{align*}
\alpha_{basis} = &c_{24} \cdot 4^{3M+4 \cdot (J+K_1)} \cdot 9^{K_1+M} \cdot M^{2 \cdot (M-1)} \cdot a^{2 \cdot (M+K_1+6)}\\
&\cdot d^{2K_1+17} \cdot (J+K_1)^{10} \cdot n^{2M+16} \cdot \max\{\max_{k,j}|\alpha_{k,j}|,1\}^{2K_1+15}\\
 &\cdot \max\{d \cdot \max_{k,j} |\alpha_{k,j}|,n\}^2 \cdot ((M+2) \cdot J + 3K_1)^2
\end{align*}
  such that
  \begin{align*}
  \left|f_{basis,i}(x) - \prod_{v=1}^J B_{j_{i,v},M, t_{i,v}}(x^{(v)}) \cdot \prod_{k=1}^K \max \left\{\sum_{j=1}^d \alpha_{i,k,j} \cdot (x^{(j)} - \gamma_{i,k,j}), 0 \right\} \right| \leq \frac{1}{n^3}
  \end{align*}
  for $x \in [-a,a]^d$.
  Set
  \begin{align*}
  f_{lcb}(x) = \sum_{i=1}^I w_i \cdot f_{basis,i}(x) \in \F^{(sparse)}_{I,L,r,\alpha}
  \end{align*}
 with $L=3K_1+J \cdot (M+2)-1$, $r=2^{M-1} \cdot 16+ \sum_{k=2}^M 2^{M-k+1}+d+5$ and \newline
 $\alpha=\max\{\alpha_{basis}, \max\{|w_1|, \dots, |w_I|\}\}$. Then we obtain
 \begin{align*}
& |f_{lcb}(x) - f(x)|\\
 \leq & \sum_{i=1}^I |w_i| \cdot \left|f_{basis,i}(x) - \prod_{v=1}^J B_{j_{i,v},M,t_{i,v}}(x^{(v)}) \cdot \prod_{k=1}^K \max \left\{\sum_{j=1}^d \alpha_{i,k,j} \cdot (x^{(j)} - \gamma_{i,k,j}), 0 \right\} \right|\\
\leq & \sum_{i=1}^I |w_i| \cdot \frac{1}{n^3}\\
\leq & I \cdot \max\{|w_1|, \dots, |w_I|\} \cdot \frac{1}{n^3}.
 \end{align*}
  \end{proof}

%-------------------------------------------Proof of Theorem 1----------------------------------

\section{Proof of Theorem \ref{th1}}
The following three auxiliary lemmatas are needed for the proof of Theorem \ref{th1}. 
\label{se4sub2}
%-----------------------------Lemma 12----------------------------------
\begin{lemma}
\label{le12}
Let $\beta_n = c_{3} \cdot \log(n)$ for some constant
$c_{3}>0$. Assume that the distribution of $(X,Y)$ satisfies (24)
for some constant $c_{4}>0$ and that the regression function $m$ is bounded in absolute value. Let $\tilde{m}_n$ be the least squares estimate $$\tilde{m}_n(\cdot)
=\arg \min_{f \in \F_n}
\frac{1}{n}
\sum_{i=1}^n
|Y_i - f(X_i)|^2$$ based on some function space $\F_n$ and set $m_n=T_{\beta_n}\tilde{m}_n$. Then $m_n$ satisfies for any $n>1$
\begin{align*}
  \mathbf E \int |m_n(x) - m(x)|^2 \PROB_X (dx) &\leq \frac{c_{29}\cdot \log(n)^2\cdot\left( \log\left(
\mathcal{N} \left(\frac{1}{n\cdot\beta_n}, \mathcal{F}_n, \|\cdot\|_{\infty,supp(X)}\right)
\right)+1  \right)}{n}\\
&\quad + 2 \cdot \inf_{f \in \mathcal F_n} \int |f(x)-m(x)|^2 \PROB_X (dx),
\end{align*}
where $c_{29}>0$ is a constant, which does not depend on $n, \beta_n$ or the parameters of the estimate.
\end{lemma}

%-----------------------------Proof of Lemma 12-------------------------------
\noindent
    {\bf Proof.}
    This lemma follows in a straightforward way from the proof of
    Theorem 1 in \cite{BaClKo09}. A complete version of the proof
    can be found in \cite{BK17}.
    \hfill $\Box$\\

\noindent
In order to bound the covering number
$\mathcal{N} \left(\frac{1}{n\cdot\beta_n}, \mathcal{F}_n,
  \|\cdot\|_{\infty,supp(X)} \right)$
we will use the following lemma.

%-------------------------------------------Lemma 13-------------------------------
\begin{lemma} \label{le13}
Let $\epsilon \geq \frac{1}{n^{c_{30}}}$ and let
$\mathcal{F}_{M^{*},L,r,\alpha}^{(sparse)}$ defined as in Section 2 with
$1 \leq \max\{a, \alpha,M^*\} \leq n^{c_{31}}$ and $L, r \leq c_{32}$ for certain
constants $c_{30}, c_{31}, c_{32} > 0$.
Assume that the squashing function $\sigma$ used in
$\mathcal{F}_{M^{*},L,r,\alpha}^{(sparse)}$
is Lipschitz continuous.
Then
\begin{align*}
\log \left(\mathcal{N}(\epsilon, \mathcal{F}_{M^{*},L,r,\alpha}^{(sparse)}, \Vert \cdot \Vert_{\infty, [-a,a]^d})\right) \leq c_{33} \cdot \log(n) \cdot M^{*}
\end{align*}
holds for any $n>1$ and a constant $c_{33} > 0$ independent of $n$ and $M^*$.
\end{lemma}

%---------------------------------Proof of Lemma 13------------------------------------
\noindent
\begin{proof}
The result follows in a straightforward way from the proof of  Lemma 9 in \cite{BK17}. For the sake of completeness we provide nevertheless the detailed proof below.

Let 
\begin{equation*}
g(x) = \sum_{k=1}^{M^*} c_k \cdot f_k^{(L+1)} \ \mbox{and} \ \bar{g}(x) = \sum_{k=1}^{M^*} \bar{c}_k \cdot \bar{f}_k^{(L+1)}(x)
\end{equation*}
with
\begin{align*}
f_k^{(L+1)} & = \sum_{i=1}^r c_{k,i}^{(L)} \cdot f_{k,i}^{(L)}(x) + c_{k,0}^{(L)}\\
\bar{f}_k^{(L+1)} & = \sum_{i=1}^r \bar{c}_{k,i}^{(L)} \cdot \bar{f}_{k,i}^{(L)}(x) + \bar{c}_{k,0}^{(L)}
\end{align*}
for some $c_{k,i}^{(L)}, \bar{c}_{k,i}^{(L)}, c_{k,0}^{(L)}, \bar{c}_{k,0}^{(L)} \in \R$ $(k=1, \dots, M^*, i=1, \dots, r)$ and for $f_{k,i}^{(L)}, \bar{f}_{k,i}^{(L)}$ recursively defined by 
\begin{align*}
f_{k,i}^{(s)}(x) &= \sigma\left(\sum_{j=1}^r c_{k,i,j}^{(s-1)} \cdot f_{k,j}^{(s-1)}(x) +c_{k,i,0}^{(s-1)}\right)\\
\bar{f}_{k,i}^{(s)}(x) &= \sigma\left(\sum_{j=1}^r \bar{c}_{k,i,j}^{(s-1)} \cdot \bar{f}_{k,j}^{(s-1)}(x) + \bar{c}_{k,i,0}^{(s-1)}\right)
\end{align*}
for some $c_{k,i,0}^{(s-1)}, \bar{c}_{k,i,0}^{(s-1)}, \dots, c_{k,i,r}^{(s-1)}, \bar{c}_{k,i,r}^{(s-1)} \in \R$ ($s=2, \dots, L$) and 
\begin{align*}
f_{k,i}^{(1)} &= \sigma\left(\sum_{j=1}^d c_{k,i,j}^{(0)} \cdot x^{(j)} + c_{k,i,0}^{(0)}\right)\\
\bar{f}_{k,i}^{(1)} &= \sigma\left(\sum_{j=1}^d \bar{c}_{k,i,j}^{(0)} \cdot x^{(j)} + \bar{c}_{k,i,0}^{(0)}\right)
\end{align*}
for some $c_{k,i,0}^{(0)}, \bar{c}_{k,i,0}^{(0)}, \dots, c_{k,i,d}^{(0)}, \bar{c}_{k,i,d}^{(0)}$. Let $C_{Lip} \geq 1$ be an upper bound on the Lipschitz constant of $\sigma$. Then 
\begin{align*}
|g(x) - \bar{g}(x)| &\leq \sum_{k=1}^{M^*} c_k \cdot |f_k^{(L+1)}(x) - \bar{f}_k^{(L+1)}(x)|\\
& \quad + \sum_{k=1}^{M^*} |c_k - \bar{c}_k| \cdot |\bar{f}_k^{(L+1)}(x)|\\
&\leq M^* \cdot \max_{k=1, \dots, M^*} |c_k| \cdot \max_{k=1, \dots, M^*} |f_k^{(L+1)}(x) - \bar{f}_k^{(L+1)}(x)|\\
& \quad + M^* \cdot \max_{k=1, \dots, M^*} |c_k- \bar{c}_k| \cdot \max_{k=1, \dots, M^*}|\bar{f}_k^{(L+1)}(x)|\\
& \leq M^* \cdot \max_{k=1, \dots, M^*} |c_k| \cdot \max_{k=1, \dots, M^*} |f_k^{(L+1)}(x) - \bar{f}_k^{(L+1)}(x)|\\
& \quad + M^* \cdot \max_{k=1, \dots, M^*} |c_k- \bar{c}_k| \cdot (r+1) \cdot \max_{\substack{k=1, \dots, M^*, \\ i=1, \dots, r}} |\bar{c}_{k,i}^{(L)}|\\
\end{align*}

From Lemma 5 in \cite{BDKKW17} we can conclude, that for any $k =1, \dots, M^*$
\begin{align*}
&|f_k^{(L+1)}(x) - \bar{f}_k^{(L+1)}(x)|\\
 &\leq (L+1) \cdot C_{Lip}^{L+1} \cdot (r+1)^{L+1} \cdot \max\{\alpha, 1\}^L \cdot \max \{\Vert x \Vert_{\infty},1\} \cdot \max_{\substack{k=1, \dots, M^*, s=0, \dots, L\\ i,j = 1, \dots, r}} \left| c_{k,i,j}^{(s)}- \bar{c}_{k,i,j}^{(s)} \right|\\
 & \leq n^{c_{34}} \cdot \max_{\substack{k=1, \dots, M^*, s=0, \dots, L\\ i,j = 1, \dots, r}} \left| c_{k,i,j}^{(s)}- \bar{c}_{k,i,j}^{(s)} \right|
\end{align*}
for $n$ sufficiently large and an adequately chosen $c_{34} > 0$ thanks to $\max\{a, \alpha, M^{*}\} \leq n^{c_{31}}$ and $L,r \leq c_{32}$.
This leads to 
\begin{align*}
|g(x) - \bar{g}(x)| &\leq  M^* \cdot \alpha \cdot n^{c_{34}} \cdot \max_{\substack{k=1, \dots, M^*, s=0, \dots, L\\ i,j = 1, \dots, r}} \left| c_{k,i,j}^{(s)}- \bar{c}_{k,i,j}^{(s)} \right|\\
& \quad + M^* \cdot \max_{k=1, \dots, M^*} |c_k- \bar{c}_k| \cdot (r+1) \cdot \alpha
\end{align*}
\begin{align*}
& \leq  n^{c_{35}} \cdot \max\left\{\max_{k=1, \dots, M^*} |c_k- \bar{c}_k|, \max_{\substack{k=1, \dots, M^*, s=0, \dots, L\\ i,j = 1, \dots, r}} \left| c_{k,i,j}^{(s)}- \bar{c}_{k,i,j}^{(s)}\right|\right\}.
\end{align*}
Thus, if we consider an arbitrary $g \in \mathcal{F}_{M^*,L,r,\alpha}^{(sparse)}$, it suffices to choose the coefficients $\bar{c}_k$ and $\bar{c}_{k,i,j}^{(s)}$ of a function $\bar{g} \in \mathcal{F}_{M^*,L,r,\alpha}^{(sparse)}$ such that 
\begin{align}\label{eq1le8}
|c_{k,i,j}^{(s)} - \bar{c}_{k,i,j}^{(s)}| \leq \frac{\epsilon}{n^{c_{35}}} \ \mbox{and} \ |c_k - \bar{c}_k| \leq \frac{\epsilon}{n^{c_{35}}},
\end{align}
which leads to $\Vert g - \bar{g} \Vert_{\infty, supp(X)} \leq \epsilon$. All coefficients are bounded by $\alpha \leq n^{c_{31}}$ and $\epsilon \geq \frac{1}{n^{c_{30}}}$, thus a number of 
\begin{align*}
\left\lceil \frac{2 \cdot \alpha \cdot n^{c_{35}}}{2 \cdot \epsilon} \right\rceil \leq n^{c_{36}}
\end{align*}
different $\bar{c}_{k,i,j}^{(s)}$ or $\bar{c}_k$ (with an equal distribution of the values in the interval $[-\alpha, \alpha]$) suffices to guarantee, that at least one of them satisfies the relation \eqref{eq1le8} for any $c_{k,i,j}^{(s)}$ or $c_k$ with fixed indices. Additionally every function $g \in \mathcal{F}_{M^*,L,r,\alpha}^{(sparse)}$ depends on
\begin{equation*}
 (d \cdot (r+1)  + L \cdot (r+1)^2 + (r+1) +1)\cdot M^* \leq c_{37} \cdot M^* 
 \end{equation*}
 different coefficients. So the logarithm of the covering number $\mathcal{N}(\epsilon, \mathcal{F}_{M^*,L,r,\alpha}^{(sparse)}, \Vert \cdot \Vert_{\infty, supp(X)})$ can be bounded by
\begin{align*}
\mathcal{N}(\epsilon, \mathcal{F}_{M^*,L,r,\alpha}^{(sparse)}, \Vert \cdot \Vert_{\infty, supp(X)}) \leq \log\left((n^{c_{36}})^{c_{37} \cdot M^*}\right) \leq c_{38} \cdot \log(n) \cdot M^*,
\end{align*}
which shows the assertion.
\end{proof}

%------------------------------Lemma 14---------------------------------------------
    \begin{lemma} \label{le14}
Let $\beta_n = c_{3} \cdot \log(n)$ for some constant $c_{3}>0$. Assume that the distribution of $(X,Y)$ satisfies (24)
for some constant $c_{4}>0$ and that the regression function $m$ is bounded in absolute value. Let $\P_n$ be a finite set of parameters, let $n=n_l+n_t$
and assume that for each $p \in \P_n$ an estimate
\[
m_{n_l,p}(x)=m_{n_l,p}(x,\D_{n_l})
\]
of $m$ is given which is bounded in absolute value by $\beta_n$.
Set
\[
\hat{p}=\arg \min_{p \in \P_n}
\frac{1}{n_t}
\sum_{i=n_{l}+1}^{n_l+n_t} |Y_i-m_{n_l,p}(X_i)|^2
\]
and define
\[
m_n(x)=m_{n_l,\hat{p}}(x).
\]
Then $m_n$ satisfies for any $n_t>1$
\begin{eqnarray*}
  &&
  \EXP\left\{ \int |m_n(x) - m(x)|^2 \PROB_X (dx) \big| \D_{n_l} \right\}
  \\
  &&
  \leq \frac{c_{39}\cdot \log(n)^2\cdot \left( \log\left( |\P_n|\right)
+1 \right)}{n_t}+ 2 \cdot \min_{p \in \P_n} \int |m_{n_l,p}(x)-m(x)|^2 \PROB_X (dx).
  \end{eqnarray*}
\end{lemma}

%-------------------------------------------Proof of Lemma 14---------------------------------
    \noindent
    \begin{proof}
      Follows by an application of Lemma \ref{le12}
        conditioned on $\D_{n_l}$. Here the covering number is
        trivially bounded by $|\P_n|$. 
        \end{proof}
\begin{proof}[Proof of Theorem \ref{th1}]
The definition of the estimate together with Lemma \ref{le14}
yields
\begin{eqnarray*}
 &&
  \EXP \int |m_n(x)-m(x)|^2 \PROB_X(dx)
 \\
&&
=
\EXP \left\{
\EXP \left\{
\int |m_n(x)-m(x)|^2 \PROB_X(dx)
\bigg| \D_{n_l}
\right\}
\right\}
\\
&&
\leq \frac{c_{39}\cdot (\log n)^2\cdot \left( \log (\left\lceil \log n
  \right\rceil)
+1 \right)}{n_t}+ 2 \cdot \min_{M^{*} \in \P_n}\EXP \int |m_{n_l,M^*}(x)-m(x)|^2 \PROB_X (dx).
\end{eqnarray*}
From Lemma \ref{le12} and Lemma \ref{le13} we conclude
\begin{eqnarray*}
&&
\EXP \int |m_{n_l,M^*}(x)-m(x)|^2 \PROB_X (dx)
\\
&&
\leq
\frac{c_{38}\cdot (\log n)^3\cdot M^*}{n_l}
+ 2 \cdot \inf_{f \in \F^{(sparse)}_{M^*, L,r,\alpha_n}} \int |f(x)-m(x)|^2 \PROB_X (dx).
\end{eqnarray*}
with $L=3K_1+d \cdot (M+2)-1$ and $r=2^{M-1} \cdot 16 + \sum_{k=2}^M 2^{M-k+1} + d+ 5$.
\newline
Combining these two results we see that 
\begin{eqnarray*}
&&
 \EXP \int |m_n(x)-m(x)|^2 \PROB_X(dx)
 \\
&&
\leq
\min_{M^* \in \{2^l \, : \, l=1,\dots, \lceil \log n \rceil\}}
\left(
\frac{c_{40}\cdot (\log n)^3\cdot M^*}{n}
+ 4 \cdot \inf_{f \in \F^{(sparse)}_{M^*, L,r,\alpha_n}} \int |f(x)-m(x)|^2 \PROB_X
   (dx)
\right).
\end{eqnarray*}
For $I \in \{1, \dots,n\}$, $a_{i} \in [-c_{10} \cdot n, c_{10} \cdot n]$
and $B_i \in \B_{n,M,K_1}^{*}$ $(i=1, \dots, I)$ set
\[
g(x)= \sum_{i=1}^I a_{i} \cdot B_i(x).
\]
Then the right--hand side of the above inequality is bounded from
above
by
\begin{eqnarray*}
&&
\min_{M^* \in \{2^l \, : \, l=1,\dots, \lceil \log n \rceil\}}
\left(
\frac{c_{40}\cdot (\log n)^3\cdot M^*}{n}
+ 8 \cdot \inf_{f \in \F^{(sparse)}_{M^*, L,r,\alpha_n}} \int |f(x)-g(x)|^2 \PROB_X
   (dx)
\right)
\\
&&
+ 8 \cdot \int |g(x)-m(x)|^2 \PROB_X
   (dx).
\end{eqnarray*}
Choose $l_I$ minimal with $2^{l_I} \geq M^*=I$, then Lemma \ref{le11} with $J=d$ and $c_i = 0$ for $i > M^*$ implies 
\begin{eqnarray*}
&&
\min_{M^* \in \{2^l \, : \, l=1,\dots, \lceil \log n \rceil\}}
\left(
\frac{c_{40}\cdot (\log n)^3\cdot M^*}{n}
+ 8 \cdot \inf_{f \in \F^{(sparse)}_{M^*, L,r,\alpha_n}} \int |f(x)-g(x)|^2 \PROB_X
   (dx)
\right)
\\
&&
\leq
\frac{c_{40}\cdot (\log n)^3\cdot 2^{l_I}}{n}
+ 8 \cdot \inf_{f \in \F^{(sparse)}_{2^{l_I}, L,r,\alpha_n}} \int |f(x)-g(x)|^2 \PROB_X
   (dx)
\\
&&
\leq
c_{40} \cdot 2 \cdot (\log n)^3 \cdot \frac{I}{n} +
8 \cdot \left(
\frac{2 \cdot I \cdot c_{10} \cdot n}{n^3}
\right)^2
.
\end{eqnarray*}
Summarizing the above results we see that we have shown
for any $I \in \{1, \dots,n\}$
\[
\EXP \int |m_{n}(x)-m(x)|^2 \PROB_X (dx)
\leq
\frac{c_{41}\cdot (\log n)^3\cdot I}{n}
+
8 \cdot \int |g(x)-m(x)|^2 \PROB_X(dx).
\]
 Since the above bound is valid for any function $g$ of the above
 form and any $I \in \{1,\dots,n\}$,
this implies the assertion.
\hfill $\Box$
\end{proof}
 
\section{Proof of Theorem 1}
\label{app4}
The proof will be divided into 4 steps.

In the {\it first step of the proof} we approximate the indicator
function of a polytope by a linear combination of some linear truncated power basis as in (20).
\newline
\newline
Let $a_i \in \R^d$ with $\Vert a_i \Vert \leq 1$, $b_i \in [-a,a]$, $\delta_i \geq c_{6} \cdot \frac{1}{n^{c_{7}}}$ and
set
\begin{equation*}
H_i=\{x \in \R^d: a_i^T x \leq b_i \},
\end{equation*}
\begin{equation*}
(H_i)_{\delta_i}=\{x \in \R^d: a_i^T x \leq b_i - \delta_i\}
\end{equation*}
and
\begin{equation*}
(H_i)^{\delta_i} = \{x \in \R^d: a_i^T x \leq b_i + \delta_i\}.
\end{equation*}
Obviously we have $(H_i)_{\delta_i} \subseteq H_i \subseteq (H_i)^{\delta_i}$. \newline
Set
\begin{equation*}
h_i(x) = \left(\frac{1}{\delta_i} \cdot \left(-a_i^T x  + b_i + \delta_i\right)\right)_+ - \left(\frac{1}{\delta_i} \cdot \left(-a_i^Tx + b_i\right)\right)_+.
\end{equation*}
So for $x \in H_i$ 
\begin{equation*}
-a_i^T x+b_i \geq 0
\end{equation*}
and  since $\delta_i > 0$
\begin{equation*}
-a_i^T x +b_i + \delta_i > 0,
\end{equation*}
this implies
\begin{equation*}
h_i(x) =1  \ \mbox{for $x \in H_i$}.
\end{equation*}
Furthermore for $x \notin (H_i)^{\delta_i}$ we know, that
\begin{equation*}
-a_i^T x+b_i+\delta_i < 0, 
\end{equation*}
and then also
\begin{equation*}
-a_i^Tx+b_i < 0. 
\end{equation*}
This leads to 
\begin{equation*}
h_i(x) = 0 \ \mbox{for $x \notin (H_i)^{\delta_i}$}.
\end{equation*}
For $x \in (H_i)^{\delta_i}\textbackslash H_i$ we can conclude, that
\begin{equation*}
-a_i^T x +b_i+\delta_i \geq 0 \ \mbox{, but} \ -a_i^T x + b_i < 0, 
\end{equation*}
which leads to 
\begin{equation*}
h_i(x) = \frac{1}{\delta_i} \cdot (-a_i^T x + b_i+ \delta_i) \in [0,1).
\end{equation*}
Therefore we can conclude, that
\begin{equation*}
1_{(H_i)_{\delta_i}}(x) \leq h_i(x) \leq 1_{(H_i)^{\delta_i}}(x) \ \mbox{for all $x \in \R^d$}.
\end{equation*}
This implies that for $\delta = (\delta_1, \dots, \delta_{K_1})$ and polytopes 
\begin{align*}
P&=\{x \in \R^d: a_i^T x \leq b_i, i=1, \dots, K_1\}\\
P_{\delta}&=\{x \in \R^d: a_i^T x \leq b_i-\delta_i, i=1, \dots, K_1\}\\
P^{\delta}&=\{x \in \R^d: a_i^T x \leq b_i+\delta_i, i=1, \dots, K_1\}
\end{align*}
 we have
\begin{equation*}
1_{P_{\delta}}(x) \leq \prod_{i=1}^{K_1} h_i(x) \leq 1_{P^{\delta}}(x) \ \mbox{for all $x \in \R^d$}.
\end{equation*}

\noindent
%\noindent Denote by $\mathcal{B}^{(trunc)}$ all functions of the form 
%\begin{equation*}
%B(x)=\prod_{k \in J} \left(\sum_{\nu =1}^d \alpha_{k,\nu} \cdot x^{(\nu)} - \gamma_k\right)_+
%\end{equation*}
%with $J \subseteq \{1, \dots, K_1\}$ $(K_1 \in \N)$, $\Vert \alpha_{k,\cdot}\Vert \leq 1$ and $\gamma_k \in [-c_6 \cdot n,c_6 \cdot n]$ for $c_6 > 0$.
We see that $\prod_{i=1}^{K_1} h_i(x)$ can be expanded in a linear combination of  $2^{K_1}$ functions of $\mathcal{B}^{*}_{n,K_1}$, if we choose $J_2=\{1, \dots, K_1\}$, $(\alpha_{k,1}, \dots, \alpha_{k,d})^T = -\frac{a_k}{\delta_k}$ and $\gamma_{k,1} = \frac{b_k+\delta_k}{a_k^{(1)}}$ or $\gamma_{k,j} = \frac{b_k}{a_{k}^{(1)}}$ and $\gamma_{k,j} = 0$ for $j >1$.
Thus we have shown: For any polytope $P=\{x \in \R^d: a_i^T x \leq b_i, i=1, \dots, K_1\}$ there exist basis functions $B^{(trunc)}_1, \dots, \B^{(trunc)}_{2^{K_1}} \in \mathcal{B}_{n,K_1}^*$ and coefficients $c_1, \dots, c_{2^{K_1}} \in \{-1,1\}$ such that
\begin{equation*}
\mathds{1}_{P_{\delta}}(x) \leq \sum_{k=1}^{2^{K_1}} c_k \cdot B^{(trunc)}_k(x) \leq \mathds{1}_{P^{\delta}}(x) \ \mbox{for all $x \in [-A,A]^d$}.
\end{equation*}
%Choose $M \geq q$, $d^* \leq d$ and $K \leq n$. For $\mathbf{j} = (j_1, \dots, j_{d^*}) \in \{-M, \dots, K-1\}^{d^*}$ and $\mathbf{t}=(t_1, \dots, t_{d^*})$ with $t_v = \{t_{v,k}\}_{k=-M, \dots, K+M}$ such that $t_{v,k} \in [-a,a]$ and $t_{v,k+1}-t_{v,k} \geq \frac{1}{n}$ $(v \in \{1, \dots, d^*\}, k \in \Z)$ 
% we define the tensor product B-spline $B_{\mathbf{j},M,\mathbf{t}}: \R^{d^*} \to \R$ by 
%\begin{equation*}
%B_{\mathbf{j},M,\mathbf{t}} = \prod_{v=1}^{d^*} B_{j_v,M,t_v}(x^{(v)})
%\end{equation*}
%with $B_{j_v,M,t_v}: \R \to \R$ as in Definition 4.

\noindent
In the {\it second step of the proof} 
we show how we can approximate
a $(p,C)$--smooth function (in case $q \leq M$)  by a linear combination of some tensor product B-spline basis, i.e. functions of the form
\begin{align*}
B_{{\mathbf j},M,\mathbf{t}}(x) := \prod_{v=1}^d B_{j_v,M,t_v}(x^{(v)})
\end{align*}
with $\mathbf{j} = (j_1, \dots, j_d) \in \{-M, -M+1, \dots, K-1\}^d$, $\mathbf{t}=(t_1, \dots, t_d)$ such that $t_v = \{t_{v,k}\}_{k=-M, \dots, K+M}$ and $t_{v,k} = -A+k \cdot \frac{A}{K}$ $(v \in \{1, \dots, d\}, k \in \Z)$ for some fixed $K \in \N$ and $B_{j_v,M,t_v}: \R \to \R$ as in Definition 4. Choose $A \geq 1$ such that $supp(X) \subseteq [-A,A]^d$. 
Let $f:\Rd \rightarrow \R$ be a $(p,C)$--smooth
function. If the spline degree $M \in \N$ fulfills the condition $M \geq q$ and we choose a knot sequence $t_{v,k} = -A + k \cdot \frac{A}{K}$ $(v \in \{1, \dots, d\}, k \in \Z)$ for some fixed $K \in \N$, standard results from the theory of B-splines (cf., e.g.,
Theorems 15.1 and 15.2 in \cite{GKKW02} and
Theorem 1 in \cite{Ko14}) imply that there exist coefficients
$b_{\mathbf j} \in \R$ which are bounded in absolute
value by some constant times $\|f\|_\infty$, such that
\[
|f(x)-\sum_{{\mathbf j} \in \{-M,\dots,K-1\}^{d}}
b_{\mathbf j} \cdot B_{{\mathbf j},M, \mathbf{t}} (x)|
\leq c_{42} \cdot \left( \frac{2A}{K} \right)^p
\quad \mbox{for all } x \in [-A,A]^{d}.
\]

In the {\it third step of the proof} we will use Theorem 1
together with the results of the previous two steps in order to show
the assertion.
Application of Theorem 1 yields
\begin{eqnarray*}
  &&
  \EXP \int |m_n(x)-m(x)|^2 \PROB_X(dx)
 \leq
  (\log n)^3 \cdot
\inf_{I \in \N, B_1, \dots, B_I \in \B_{n,M,K_1}^{*}}
\Bigg(
c_9 \cdot \frac{I}{n}
  \\
  &&
\hspace*{3cm}
+
\min_{(a_k)_{k=1,\dots, I} \in [-c_{10} \cdot n,c_{10} \cdot n]^I}
\int
|\sum_{k=1}^I a_k \cdot B_k(x)-m(x)|^2 \PROB_X(dx)
\Bigg).
\end{eqnarray*}
Hence it suffices to show that there exist $J\in \N$,
$a_k \in [- c_{10} \cdot n, c_{10} \cdot n]$ and $B_k \in \B_{n,M, K_1}^{*}$ $(k=1,\dots,J)$ such that
\[
\frac{J}{n}
+
\int
|\sum_{k=1}^J a_k \cdot B_k(x)-m(x)|^2 \PROB_X(dx)
\leq c_{43} \cdot n^{-\frac{2p}{2p+d^*}}.
\]
By the assumption of the theorem there exist
$K_2 \in \N$, polytopes $P_1$, \dots, $P_{K_2} \subset \Rd$,
$(p,C)$--smooth and bounded functions
$f_1$, \dots, $f_{K_2}:\R^{d^*} \rightarrow \R$ and subsets
$J_1$, \dots, $J_{K_2} \subset \{1, \dots, d\}$ of cardinality at most
$d^*$ such that
\[
\sum_{k=1}^{K_2}
f_k(x_{J_k}) \cdot 1_{
(P_k)_{\delta_k}
}(x)
\leq
m(x)
\leq
\sum_{k=1}^{K_2}
f_k(x_{J_k}) \cdot 1_{
(P_k)^{\delta_k}
}(x)
\]
holds for all $x \in [-A,A]^d$.

First we show that we can approximate $m(x)$ by a linear combination of our basis function in (18) in case 
\[
x \in \Rd \setminus
\left(
  \left(\cup_{k=1}^{K_2}(P_k)^{\delta_k} \textbackslash (P_k)_{\delta_k}\right)\cap [-A,A]^d
\right).
\]
Here we have
\[
m(x) = \sum_{k=1}^{K_2} f_k(x_{J_k}) \cdot 1_{(P_k)_{\delta_k}}(x).
\]
By the first step of the proof there exist $c_{k,j} \in \{-1,1\}$ $(k=1, \dots, K_2, j=1, \dots, 2^{K_1})$ and \newline
 $B^{(trunc)}_{k,1}, \dots, B^{(trunc)}_{k,{2^{K_1}}} \in \B_{n,K_1}^*$, such that
\[
\sum_{j=1}^{2^{K_1}} c_{k,j} \cdot B_{k,j}^{(trunc)}(x) = 1 \ \mbox{for} \ x \in (P_k)_{\delta_k}
\]
and 
\[
\sum_{j=1}^{2^{K_1}} c_{k,j} \cdot B_{k,j}^{(trunc)}(x) = 0 \ \mbox{for} \ x \notin (P_k)^{\delta_k}
\]
for some $k \in \{1, \dots, K_2\}$. 
By the second step of the proof each $f_k(x_{J_k})$ $(k=1, \dots, K_2)$ can be approximated by a linear combination of a tensor product B-Spline $(B_{{\mathbf j},M,{\mathbf t}})_{\mathbf{j} \in \{-M,-M+1, \dots, K-1\}^{d^*}}$ with $\mathbf{t}$ chosen as in the second step. Remark that we replace $d$ by $d^*$, since $f_k$ depends only on a maximum of $d^*$ input coefficients. 
%Together we can conclude that for +
Set
\begin{align*}
\bar{m}(x)& = \sum_{k=1}^{K_2} \left(\sum_{j=1}^{2^{K_1}} c_{k,j} \cdot B_{k,j}^{(trunc)}(x)\right) \cdot \left(\sum_{\mathbf{j} \in \{-M,-M+1, \dots, K-1\}^{d^*}} b_{k,{\mathbf j}} \cdot B_{{\mathbf j},M,{\mathbf t}}(x) \right)\\
&= \sum_{k=1}^{K_2 \cdot 2^{K_1} \cdot (M+K)^{d^*}} \tilde{c}_k \cdot B_k(x)
\end{align*}
with 
%$B_{k,1}^{(trunc)}, \dots, B_{k,2^{K_1}}^{(trunc)} \in \B_{n,K_1}^*$ $(k=1, \dots, K_2)$, $c_{k,j} \in \{-1,1\}$ $(k=1, \dots, K_2, j=1, \dots, 2^{K_1})$, 
\begin{align*}
 &|b_{k,{\mathbf{j}}}| \leq c_{43} \cdot \max_k \Vert f_k \Vert_{\infty} \ (k=1, \dots, K_2, \mathbf{j} \in \{-M,-M+1, \dots, K-1\}^{d^*}),\\
 &B_1, \dots, B_{K_2 \cdot 2^{K_1} \cdot (M+K)^{d^*}} \in \B_{n,M,K_1}^*,\\
&\tilde{c}_k \in [-c_{43} \cdot \max_k \Vert f_k \Vert_{\infty}, c_{43} \cdot \max_k \Vert f_k \Vert_{\infty}] \ (k=1, \dots, K_2 \cdot 2^{K_1} \cdot (M+K)^{d^*}).
\end{align*}
Then it follows
\begin{align*}
|\bar{m}(x)-m(x)| \leq c_{44} \cdot K_2 \cdot \left(\frac{2A}{K}\right)^p
\end{align*}
for 
\[
x \in \Rd \setminus
\left(
  \left(\cup_{k=1}^{K_2}(P_k)^{\delta_k} \textbackslash (P_k)_{\delta_k}\right)\cap [-A,A]^d
\right).
\]
Now we choose $K = \lceil n^{\frac{1}{2p+d^*}}\rceil$ and set $J=K_2 \cdot 2^{K_1} \cdot (c_{45} \cdot \lceil n^{\frac{1}{2p+d^*}} \rceil)^{d^{*}}$ with $c_{45} > 0$ suitably large. With the previous result we can conclude that there exist $B_1$, \dots, $B_J \in \B_{n,M,K_1}^{*}$ and
$\gamma_1, \dots, \gamma_J \in [-c_{43} \cdot \max_k \Vert f_k \Vert_{\infty}, c_{43} \cdot \max_k \Vert f_k \Vert_{\infty}]$
%(where $c_{40}$ is bounded by some constant times the maximum
%supremum norm bound on $f_1$,\dots,$f_{K_2}$)
 such that
for any
\[
x \in \Rd \setminus
\left(
  \left(\cup_{k=1}^{K_2}(P_k)^{\delta_k} \textbackslash (P_k)_{\delta_k}\right)\cap [-a,a]^d
\right)
\]
the following inequality holds:
\begin{align*}
|\sum_{j=1}^J \gamma_j \cdot B_j(x)-m(x)| &\leq
c_{46} \cdot K_2 \cdot
\left(
\frac{1}{\lceil n^{\frac{1}{2p+d^*}} \rceil}
\right)^p.
\end{align*}
Additionally, we can conclude, that
\begin{align*}
|\sum_{j=1}^J \gamma_j \cdot B_j(x)| & \leq \sum_{k=1}^{K_2} \left|\sum_{\mathbf{j} \in \{-M, -M+1, \dots, K-1\}^{d^*}} b_{k,\mathbf{j}} \cdot B_{\mathbf{j},M,\mathbf{t}}(x)\right|\\
&\leq \sum_{k=1}^{K_2} \left|\sum_{\mathbf{j} \in \{-M, -M+1, \dots, K-1\}^{d^*}} b_{k,\mathbf{j}} \cdot B_{\mathbf{j},M,\mathbf{t}}(x)-f_k(x_{J_k})\right| + \max_k \Vert f_k \Vert_{\infty} \cdot K_2\\
&\leq c_{47} \cdot K_2.
\end{align*}
This together with the assumed $\PROB_X$-border $\frac{c_{7}}{n}$ of the Theorem implies
\begin{eqnarray*}
&&
\frac{J}{n}
+
\int
|\sum_{k=1}^J a_k \cdot B_k(x)-m(x)|^2 \PROB_X(dx)
\\
&&
\leq
 c_{48} \cdot 2^{K_1} \cdot K_2 \cdot n^{-\frac{2p}{2p+d^*}} + c_{49}
\cdot K_2 \cdot 
\PROB_X \left(
\left( \bigcup_{k=1}^{K_2}(P_k)^{\delta_k} \textbackslash (P_k)_{\delta_k}\right)\cap [-a,a]^d
\right)
\\
&&
\leq
 c_{48}\cdot 2^{K_1} \cdot K_2 \cdot  n^{-\frac{2p}{2p+d^*}}+ c_{49}
\cdot K_2 \cdot  \frac{c_{7}}{n}
   \leq c_{50} \cdot 2^{K_1} \cdot K_2 \cdot n^{-\frac{2p}{2p+d^*}}.
\end{eqnarray*}
   \quad \hfill $\Box$

\end{document}